%% file: main.tex
\documentclass{article}
\usepackage{pgfplots}
\usepackage{tikz}
\usepackage{multirow}
\usepackage{xr}
\usepackage{tcolorbox}

\usepackage[main,preprint]{neurips_2026}
\makeatletter\renewcommand{\@noticestring}{\@trackname}\makeatother

\usepackage[utf8]{inputenc} %
\usepackage[T1]{fontenc}    %
\usepackage{hyperref}       %
\usepackage{url}            %
\usepackage{booktabs}       %
\usepackage{amsfonts}       %
\usepackage{nicefrac}       %
\usepackage{microtype}      %
\usepackage{xcolor}         %
\usepackage{amsmath,amssymb,amsthm,mathtools,bm}
\usepackage{algorithm}
\usepackage{algpseudocode}

\usepackage{subcaption}

\usepackage{environ}
\usetikzlibrary{fit,arrows.meta,calc}
\newcommand{\dBern}{d_{\mathrm{Bern}}}   %

\newtheorem{theorem}{Theorem}
\newtheorem{lemma}{Lemma}

\newtheorem{assumption}{Assumption}
\theoremstyle{definition}
\newtheorem{definition}{Definition}
\theoremstyle{remark}
\newtheorem{remark}{Remark}

\DeclareMathOperator*{\argmax}{arg\,max}
\DeclareMathOperator*{\argmin}{arg\,min}

\DeclareMathOperator{\logit}{logit}
\DeclareMathOperator{\Gap}{Gap}
\DeclareMathOperator{\clip}{clip}
\DeclareMathOperator{\proj}{Proj}

\newcommand{\KL}{D_{\mathrm{KL}}}
\newcommand{\BKL}{D_{\mathrm{KL}}^{\mathrm{Bern}}}
\newcommand{\Nbeta}{\mathcal N_\beta}

\title{Robust Nash Alignment under Preference Uncertainty}

\author{%
  Shihab Ahmed \\
  University of Central Florida \\
  \texttt{shihab.ahmed@ucf.edu} \\
  \And
  Debamita Ghosh \\
  University of Central Florida \\
  \texttt{debamita.ghosh@ucf.edu} \\
  \And
  David Tang \\
  University of Central Florida \\
  \texttt{david.tang@ucf.edu} \\
  \And
  Yudan Wang \\
   University of Central Florida \\
  \texttt{yudan.wang@ucf.edu} \\
  \And
  Alvaro Velasquez \\
  University of Colorado Boulder \\
  \texttt{alvaro.velasquez@colorado.edu} \\
  \And
  Yue Wang \\
  University of Central Florida \\
  \texttt{yue.wang@ucf.edu} \\
}

\pgfplotsset{compat=1.18}
\begin{document}

\maketitle

\begin{abstract}
Preference-based alignment methods typically optimize against a single preference model, and can therefore be brittle when pairwise preferences are uncertain: noisy, heterogeneous, or shift after deployment. To address these issues,
we propose Robust Nash Alignment, a game-theoretic framework for alignment to uncertain pairwise preferences. Our formulation has a major learner seeking  a policy with a large worst-case win rate against both an adversarial competitor and any preference kernel lying in an ambiguity set around a nominal preference. When the ambiguity set captures the uncertainty in preferences, the resulting robust objective of the game directly yields a certified lower bound on worst-case performance. However, we note this problem is computationally challenging to optimize, and to address this, we introduce a four-player primal-dual proxy game involving the leader policy, follower policy, adversarial kernel, and dual variable, and develop a single-loop optimistic mirror descent-ascent algorithm for it. We show that the proxy always lower-bounds the truncated hard-constrained objective, quantify the proxy-to-hard gap, and characterize an exactness condition under which the proxy recovers the robust objective. We then prove an \(\mathcal{O}(1/\sqrt{T})\) average-iteration convergence for the proxy-game duality gap, which implies a near-optimal robust policy for the original robust objective. Experiments on controlled tabular games and LLM alignment with uncertain preference illustrate the method and show improved robustness to preference perturbations over nominal baselines.
\end{abstract}

\section{Introduction}
Aligning large language models (LLMs) with preferences is a central challenge in modern post-training.
The dominant paradigm, reinforcement learning from human feedback (RLHF)~\citep{christiano2017deep,stiennon2020learning,ouyang2022training}, collects pairwise comparisons, fits a scalar reward model, and optimizes policy against that reward with a proximity constraint to a reference model.
More recent direct alignment methods, including Direct Preference Optimization (DPO), SLiC-HF, Contrastive Preference Learning (CPL), and Soft Preference Optimization (SPO)~\cite{rafailov2023direct,zhao2023slic,hejna2023contrastive,sharifnassab2024soft}, bypass the explicit reward model by optimizing preference-derived surrogate losses directly.
Despite their practical success, all of these methods share a common structural assumption: they model preferences through a latent scalar reward function, typically via the Bradley-Terry (BT)~\citep{thurstone1927psychophysical,bradley1952rank} or Thurstone framework.
This assumption is, however, restrictive.
Preferences may be context-dependent, cyclic, or heterogeneous across annotators, and cannot always be represented by a single scalar reward~\citep{azar2024general,munos2023nash,Zhang2025INPO,ye2024online,chidambaram2025direct,furuta2024geometric}.

A growing line of work addresses this limitation by modeling preferences directly through pairwise comparison probabilities and framing alignment as a two-player zero-sum game.
In particular, Nash Learning from Human Feedback (NLHF) and its extensions~\citep{munos2023nash,Zhang2025INPO,ye2024online,calandriello2024human} seek a policy that is hard to beat in the induced preference game, rather than one that maximizes a learned scalar reward.
This general-preference perspective resolves the modeling mismatch: it can represent non-transitive, cyclic, and population-heterogeneous preferences that no scalar reward can capture.
However, it introduces a different and equally important challenge that has been largely overlooked:
\emph{the preference kernel itself is uncertain.}

The true pairwise preference probability $P(y_i \succ y_j \mid x)$ is never observed or fixed.
In practice, it must be estimated from finite, noisy, and potentially biased comparisons, often aggregated across heterogeneous annotator populations and changing deployment environments~\citep{christiano2017deep,stiennon2020learning,ouyang2022training,tang2024generalized}.
Existing general-preference methods, including NLHF, treat the learned preference kernel as a fixed ground truth and optimize against this single point estimate.
This can be brittle:
small errors in preference estimation, misspecification of the comparison model, or deployment-time shifts in who provides feedback may materially change which responses the policy favors~\citep{rahimian2019distributionally,sagawa2020distributionally,chidambaram2025direct,furuta2024geometric}.
We refer to this vulnerability as the \emph{fragile consensus} problem---a policy that appears optimal under the nominal preference data may concentrate on responses whose advantages rest on slim margins that are statistically indistinguishable from noise, and a small perturbation to the preference kernel can reverse the dominance ordering entirely.

To address this challenge, we propose to treat the preference kernel as fundamentally uncertain and to optimize for the worst case over a set of plausible kernels.
Inspired by the literature on distributionally robust optimization~\citep{rahimian2019distributionally,iyengar2005robust,nilim2005robust}, where ambiguity sets, divergence balls, and worst-case guarantees are standard tools for hedging estimation error and distribution shift, we formulate \emph{alignment under uncertain preferences} as a robust two-player game over a pairwise preference kernel.
At each prompt, the main player seeks to maximize a Kullback-Leibler (KL)-regularized win rate, while an adversarial player simultaneously chooses the hardest competitor \emph{and} the worst-case preference kernel from an ambiguity set centered at the nominal estimate.
The resulting objective directly provides a lower bound on the policy's alignment performance under all plausible preference kernels in the ambiguity set, effectively addressing the fragile consensus problem.

Our contributions are summarized as follows.

\textbf{Robust preference-game formulation.}
We formulate alignment under preference uncertainty as a robust game over a pairwise preference kernel with a Bernoulli-KL ambiguity set.
Unlike standard NLHF, which optimizes against a single nominal kernel, our objective maximizes the worst-case game value over all kernels within a prescribed divergence ball. We show that this formulation yields a principled worst-case performance guarantee: the learned policy's win rate is lower-bounded against every plausible preference kernel in the ambiguity set, improving the robustness and reliability of preference-based alignment.

 \textbf{Four-player algorithm with convergence guarantees.}
To avoid the computational burden of solving the hard-constrained distributionally robust problem directly, we pass to its Lagrangian as a proxy and derive a four-player primal-dual formulation involving the leader policy, follower policy, adversarial preference kernel, and dual variable.
We further provide a detailed duality and convergence analysis: we characterize when the four-player game exactly or approximately recovers the original hard-constrained robust objective, and we prove that the proposed algorithm converges to a saddle point of the proxy robust game at a rate of $\mathcal{O}(1/\sqrt{T})$.
Under the stated exactness conditions, this yields a certified worst-case game-value guarantee.

  \textbf{Experimental evidence.}
We evaluate the method on controlled tabular games, a best-of-$m$ reranking task on IMDb continuations with heterogeneous annotator groups, and a summarization task on TL;DR .
At the matched radius the robust policy improves the worst-case value over the non-robust NashMD baseline on every tabular prompt, it is preferred to NashMD under most perturbed kernels in the IMDb task and under external LLM judges in five of six summarization configurations, while maintaining comparable nominal performance. These experiments illustrate the method and verify no guarantee.

\textbf{Related work.}
Distributionally robust optimization hedges estimation error and shift with ambiguity sets \citep{sagawa2020distributionally,iyengar2005robust,nilim2005robust}, robust alignment methods robustify reward-based or DPO-style objectives \citep{chakraborty2024maxmin,mandal2025distributionally}, and robust (Markov) game theory treats payoff or transition uncertainty inside equilibrium concepts \citep{aghassi2006robust,liu2018distributionally,zhang2020robust}.
We instead robustify the pairwise kernel itself inside a general-preference game, with a coupled, non-rectangular budget and an asymmetric leader--follower objective (Appendix~\ref{app:related_work}).

\section{Preliminaries}
\label{sec:preliminaries}
\textbf{Notation.}
Let \(\mathcal X\) denote the prompt space, and let \(D_0\) be the deployment-time prompt
distribution. For each prompt \(x\in\mathcal X\), let \(Y(x)=\{y_1,\dots,y_m\}
\)
be the finite response set.
A policy is a mapping \(\pi:\mathcal X\to \Delta(Y)\), where \(\Delta(Y(x))\) denotes the simplex
over responses for prompt \(x\). Given \(x\), a response \(y\) is sampled from \(\pi(\cdot\mid x)\).
We also fix a reference policy \(\pi_{\rm ref}(\cdot\mid x)\) with full support, which can be derived from Supervised Fine-Tuning (SFT), and a KL
regularization parameter \(\tau>0\).

\textbf{Reward-based alignment.}
A standard approach to preference-based alignment assumes that pairwise human preferences are
generated from an underlying scalar reward function \(R^\star(x,y)\). Under the classical
Bradley-Terry / logistic preference model
\citep{bradley1952rank,christiano2017deep,ouyang2022training},
the probability that response \(y\) is preferred to \(y'\) under prompt \(x\) is
\begin{equation}
\mathbb P(y \succ y' \mid x)
=
\frac{\exp(R^\star(x,y))}
{\exp(R^\star(x,y))+\exp(R^\star(x,y'))}.
\label{eq:bt-model}
\end{equation}

Given a dataset with pairwise responses and preferences, one can first learn a reward model $R$ through regression, and maximize the classical KL-regularized RLHF objective
\begin{equation}
J_{\rm RB}(\pi;R)
:=
\mathbb E_{x\sim D_0}
\left[
\mathbb E_{y\sim \pi(\cdot\mid x)}[R(x,y)]
-
\tau \KL\bigl(\pi(\cdot\mid x)\,\|\,\pi_{\rm ref}(\cdot\mid x)\bigr)
\right].
\label{eq:rb-objective}
\end{equation}
The reward-based alignment problem is then $\pi_R^\star
\in
\arg\max_{\pi\in\Pi}
J_{\rm RB}(\pi;R),$
where \(\Pi\) is the policy class of interest. The KL term discourages the learned policy from drifting arbitrarily far from the reference policy
\citep{stiennon2020learning,ouyang2022training,rafailov2023direct}, and it results in a closed-form solution to  \eqref{eq:rb-objective}  as
\begin{equation}
\pi_R^\star(y\mid x)
\propto
\pi_{\rm ref}(y\mid x)\exp\left(\frac{1}{\tau}R(x,y)\right),
\label{eq:rb-closed-form}
\end{equation}
which underlies the derivation of the direct preference optimization method
\citep{rafailov2023direct}.

\paragraph{General preference alignment.}
A limitation of \eqref{eq:bt-model} is that it assumes all pairwise preferences are induced by
a single scalar reward function. More generally, one can directly treat pairwise preference
probabilities as primitive objects, without assuming the latent
reward \citep{munos2023nash,ye2024online,Zhang2025INPO}.

\begin{definition}
\label{def:general-preference-oracle}
A general preference oracle is a mapping
\(
P:\mathcal X \times Y \times Y \to [0,1]
\)
such that, for every prompt \(x\in\mathcal X\) and responses \(y,y'\in Y(x)\), there exists a preference $z \sim \mathrm{Ber}\bigl(P(y\succ y' \mid x)\bigr),$ where \(z=1\) indicates that \(y\) is preferred to \(y'\), and \(z=0\) indicates the opposite.
\end{definition}
\begin{assumption}[Antisymmetric preference model]
\label{ass:antisymmetry}
For any fixed prompt \(x\), define \(P_x(y \succ y') := P(y \succ y' \mid x)\).
We assume that the preference probabilities are antisymmetric, i.e.,
\begin{equation}
P_x(y_i \succ y_j) + P_x(y_j \succ y_i) = 1,
\qquad \forall\, y_i \neq y_j.
\end{equation}
\end{assumption}

Given two policies \(\pi,\nu\in\Pi\), Nash learning from human feedback (NLHF) considers the KL-regularized game payoff under the general
preference kernel \(P\):
\begin{align}
J(\pi,\nu;P)
&:=
\mathbb E_{x\sim D_0}
\Big[
\mathbb E_{y\sim\pi(\cdot\mid x),\; y'\sim \nu(\cdot\mid x)}
\big[P_x(y \succ y')\big]
-
\tau \KL\bigl(\pi(\cdot\mid x)\,\|\,\pi_{\rm ref}(\cdot\mid x)\bigr)\nonumber\\
&\quad+
\tau \KL\bigl(\nu(\cdot\mid x)\,\|\,\pi_{\rm ref}(\cdot\mid x)\bigr)
\Big].
\label{eq:general-pref-game}
\end{align}
Here, \(\pi\) is the max-player and \(\nu\) is the min-player. The interpretation is that
both players seek a policy that is consistently preferred to any competitor, while both
players are regularized toward the same reference policy. The nominal general-preference alignment problem is therefore formulated as a regularized minimax game
\begin{equation}
(\pi^\star,\nu^\star)
\in
\arg\max_{\pi\in\Pi}\arg\min_{\nu\in\Pi} J(\pi,\nu;P).
\label{eq:general-pref-minimax}
\end{equation}
Due to the symmetric and convex-concave game structure, \eqref{eq:general-pref-minimax} has a unique symmetric Nash equilibrium  \((\pi^\star,\pi^\star)\) \citep{munos2023nash}. NLHF then aims to find this policy $\pi^\star$ as the alignment policy.

\section{Robust Game Formulation for Uncertain Preference Alignment}\label{sec:robust-game}
In this section, we introduce our formulation of a robust game to address uncertainty in preferences. For notational simplicity, we fix a prompt and suppress the prompt index. The global prompt-averaged objective is obtained by taking expectations under a prompt-rectangular ambiguity model.

Let \(Y=\{1,\dots,m\}\) be a finite response set, and \(d=\binom{m}{2}\), and set $E:=\{(i,j):1\le i<j\le m\}.$ For \(p=(p_{ij})_{(i,j)\in E}\in[0,1]^d\), we define the preference kernel that satisfies Assumption \ref{ass:antisymmetry} as
\[
P_p(i\succ j)
=
\begin{cases}
p_{ij}, & i<j,\\
1-p_{ji}, & i>j,\\
0.5, & i=j.
\end{cases}
\]
Here $i<j$ is the ordinary index order: each unordered pair is listed once, and the reverse comparison follows by antisymmetry.

We model the potential uncertainty in the preference kernel through the Bernoulli-KL divergence. Namely, centered at a nominal kernel \(p^\star\in(0,1)^d\), the Bernoulli-KL divergence is defined as
\begin{equation}
\BKL(p\|p^\star)
:=
\sum_{(i,j)\in E}
\left[
p_{ij}\log\frac{p_{ij}}{p^\star_{ij}}
+
(1-p_{ij})\log\frac{1-p_{ij}}{1-p^\star_{ij}}
\right],
\label{eq:bern-kl}
\end{equation}
measuring the difference between the two kernels. We then set the ambiguity set as
\begin{align}\label{eq:amb set}
    \mathcal{P}:=\{ p: \BKL(p\|p^\star)\leq \rho \},
\end{align}
where $\rho$ is some uncertainty radius, quantifying the uncertainty level. Namely, the ambiguity set contains the kernels whose differences from the nominal one are bounded by $\rho$.

\begin{remark}
The ambiguity set $\mathcal{P}$ captures preference kernels that are close to a nominal estimate $p^\star$, and thus reflects uncertainty arising from finite, noisy, or heterogeneous preference data. The radius $\rho$ is generally chosen by domain experts or users, reflecting the users' consideration level of uncertainty.
For instance, in offline preference-based learning, one may estimate $p^\star$ from a dataset of pairwise comparisons and choose the radius $\rho$ via concentration inequalities \citep{zhu2023principled} so that $\mathcal{P}$ contains all statistically plausible kernels; When preferences are aggregated across different users or annotator populations, $\mathcal{P}$ can model the resulting variability in pairwise comparisons.

Unlike standard robust optimization settings where uncertainty sets are often rectangular and decouple across components \citep{wiesemann2013robust,iyengar2005robust}, our set $\mathcal{P}$ induces coupling across all pairwise comparisons. As a result, the set is inherently non-rectangular, which significantly complicates the problem and is a key source of computational hardness in our formulation.
\end{remark}

\paragraph{Objective.} With the ambiguity set defined in \eqref{eq:amb set}, we propose our robust formulation as
\begin{align}
\label{eq:obj}
   \max_{\pi\in\Delta_m} \min_{\nu\in\Delta_m}
\min_{p\in\mathcal{P}}
J(\pi,\nu;p),
\end{align}
where $J(\pi,\nu;p)$ is the game value under the specific kernel $p$:
\begin{equation}
J(\pi,\nu;p)
:=
\sum_{i=1}^m\sum_{j=1}^m \pi_i \nu_j\, P_p(i\succ j)
-
\tau \KL(\pi\|\pi_{\mathrm{ref}})
+
\tau \KL(\nu\|\pi_{\mathrm{ref}}).
\label{eq:J-def}
\end{equation}
Namely, we enable the adversarial player to further choose a preference kernel as part of its strategy, and the main player needs to maximize its winning rate under the worst-case, i.e., the least favorable preference kernel. This consideration of the worst-case then provides a lower bound guarantee. Namely, if $\pi^\star$ solves \eqref{eq:obj}, then for any $p\in\mathcal{P}$ and any adversarial policy $\nu$, one has (ignoring the regularizer) an optimized lower bound on the winning rate under $p$ and against $\nu$:
\begin{align}
    P_p(\pi^\star \succ \nu)\geq \min_{p\in\mathcal{P},\nu}  P_p(\pi^\star \succ \nu) = \max_\pi \min_{p\in\mathcal{P},\nu} P_p(\pi \succ \nu)=\text{ optimal value of \eqref{eq:obj}}.\label{eq:certificate}
\end{align}
Therefore, the maximizer $\pi^\star$ of our robust value \eqref{eq:obj} provides a baseline performance of uncertain alignment, thus improving the robustness and reliability against preference uncertainty.

\section{Four-Player Proxy Game}
In this section, we aim to design a concrete algorithm to solve the robust game \eqref{eq:obj}.

For standard, non-robust NLHF, where the kernel is fixed and the only strategic interaction is between the two players, the optimization landscape is that of a regularized two-player game, which can be handled by no-regret and optimistic first-order methods. However, our robust formulation is qualitatively more difficult. The major challenge arises from the fact that, even for a fixed leader policy \(\pi\), the inner value $g_H(\pi;\rho)
:=
\min_{\nu,p\in\mathcal P} J(\pi,\nu;p)$ ($\rho$ enters only through the feasible set $\mathcal P$) is difficult to compute exactly because \(\nu\) and \(p\) are coupled through the win-rate term.
If one first fixes $\pi, p$ and eliminates the adversarial policy $\nu^\star$ using the Gibbs response, the reduced objective becomes
\begin{equation}
\min_{p\in \mathcal{P}}
\Big\{
-\tau \KL(\pi\|\pi_{\rm ref})
-
\tau \log \sum_{j=1}^m
\pi_{\rm ref}(j)
\exp\big(
-\frac{b(j;\pi,p)}{\tau}
\big)
\Big\},
\label{eq:reduced-concave-min}
\end{equation}
where \(b(j;\pi,p)=\sum_{i=1}^m \pi(i\mid x)P_p(i\succ j)\) is affine in \(p\). Since the log-sum-exp is convex and appears with a
negative sign, the reduced objective is concave in \(p\). However, minimizing a concave function over a
convex uncertainty set is NP-hard in general
\citep{horst2013global}. On the other hand, if one fixes $\pi,\nu$ first and solves for the worst-case kernel $p^\star$, since the uncertainty set $\mathcal{P}$ is defined with correlation among response pairs, $\mathcal{P}$ is not rectangular over all $(i,j)$ pairs; thus, solving the distributionally robust optimization problem can still be inefficient or NP-hard  \citep{wiesemann2013robust}.
Moreover, standard NLHF involves only two strategic blocks \((\pi,\nu)\), yet our robust game has another player controlling $p$. Although one can combine the two minimizing players as a single adversary who takes the joint action $(\nu,p)$, it breaks the standard two-player saddle problem structure in NLHF, losing all of its properties like the minimax theorem or unique Nash equilibrium \citep{munos2023nash}. The resulting game need not be monotone in the Euclidean sense, which makes naive
simultaneous gradient methods unstable.

Due to these challenges, we develop a proxy game formulation that is easier to solve and develop detailed studies on the connection between the proxy and the original robust game.

For regularity, we first assume the main player policy is constrained to the floor-truncated simplex $\Pi_\epsilon
:=
\left\{
\pi\in\Delta_m:\ \pi_i\ge \epsilon\,\pi_{\mathrm{ref},i}\ \ \forall i
\right\}$ for some positive value $\epsilon$. Specifically, one can choose $\epsilon$ arbitrarily small to approximately recover the probability simplex $\Delta_m$. For each \(\pi\in\Pi_\epsilon\), denote the original inner problem value with the hard constraint $p\in\mathcal{P}$ as
\begin{equation}
g_H(\pi;\rho)
:=
\min_{\nu\in\Delta_m}
\min_{p\in\mathcal{P}}
J(\pi,\nu;p),
\label{eq:hard-original}
\end{equation}
and the corresponding optimal robust value as
\begin{equation}
V(\rho):=\max_{\pi\in\Pi_\epsilon} g_H(\pi;\rho).
\label{eq:V-original}
\end{equation}

We further define an interior box
$P_\delta:=[\delta,1-\delta]^d$.
with some \(\delta\in(0,1/2)\) due to regularity, and  consider the truncated hard objective:
\begin{equation}
g_H^\delta(\pi;\rho)
:=
\min_{\nu\in\Delta_m}
\min_{\substack{p\in P_\delta \cap \mathcal{P}}}
J(\pi,\nu;p),
\label{eq:hard-truncated}
\end{equation}
and denote its optimal value as
\begin{equation}
V^\delta(\rho):=\max_{\pi\in\Pi_\epsilon} g_H^\delta(\pi;\rho).
\label{eq:V-truncated}
\end{equation}

We first show that the truncation error is explicit and uniformly controlled.
\begin{theorem}[Uniform approximation by the truncated hard objective]
\label{thm:truncate}
Assume \(p^\star\in P_\delta\). Then, for every \(\pi\in\Pi_\epsilon\), $0
\le
g_H^\delta(\pi;\rho)-g_H(\pi;\rho)
\le
\delta$ and  $0
\le
V^\delta(\rho)-V(\rho)
\le
\delta.$ Moreover, if \(\pi_\delta^\star\in \argmax_{\pi\in\Pi_\epsilon} g_H^\delta(\pi;\rho)\), then
\(
g_H(\pi_\delta^\star;\rho)\ge V(\rho)-\delta.
\)
\end{theorem}
This theorem shows that the error introduced by the regularity relaxation is bounded by $\delta$. Thus, one can choose $\delta$ small enough to ensure regularity and approximate the original solution arbitrarily. Therefore, we will focus on solving this problem, $\arg\max_{\pi\in\Pi_\epsilon} g_H^\delta(\pi;\rho)$. Since $g^\delta_H$ is a hard-constrained problem, which can be inefficient to solve directly, we introduce a Lagrangian dual variable and consider the Lagrangian function. Specifically, with an additional dual variable $\gamma$, we define and consider the four-player payoff as
\begin{equation}
F(\pi,\gamma,\nu,p)
:=
J(\pi,\nu;p)
+
\gamma\bigl(\BKL(p\|p^\star)-\rho\bigr).
\label{eq:F-def}
\end{equation}
Moreover, fix a compact positive interval $\Gamma:=[\gamma_{\min},\gamma_{\max}]
\subset (0,\infty)$ and $\Nbeta
:=
\left\{
\nu\in\Delta_m:\ \nu_j\ge \beta\,\pi_{\mathrm{ref},j}\ \ \forall j
\right\}$ for $\beta \in (0,1)$, we define the proxy objective as
\begin{equation}
g_{\mathrm{proxy}}^\delta(\pi;\rho)
:=
\max_{\gamma\in\Gamma}
\min_{\nu\in\Nbeta,\;p\in P_\delta}
F(\pi,\gamma,\nu,p),\footnote{We proved that restricting the opponent policy from $\Delta_m$ to $\mathcal{N}_\beta$ does not affect the optimal solution.}
\label{eq:proxy-def}
\end{equation}
and its optimal value as
\begin{equation}
V_{\mathrm{proxy}}^\delta(\rho)
:=
\max_{\pi\in\Pi_\epsilon} g_{\mathrm{proxy}}^\delta(\pi;\rho).
\label{eq:proxy-value}
\end{equation}
We then study the proxy's connection to the original problem. Firstly, we introduce a regularity assumption.
\begin{assumption}
\label{ass:geometry}
We assume the following regularization conditions hold: (1)  \(p^\star\in(\delta,1-\delta)^d\);
    (2) \(0<\beta\le e^{-1/\tau}\);
    (3) \(\Gamma=[\gamma_{\min},\gamma_{\max}]\subset(0,\infty)\) with
    \(\gamma_{\max}\ge 1/\rho\);
    (4).  \(\tau\gamma_{\min}\ge \frac14\).
\end{assumption}
Assumption \ref{ass:geometry}-(1)–(3) are mild regularity conditions that can be enforced by suitable algorithmic
choices. Assumption \ref{ass:geometry}-(4) is a curvature condition needed for the proxy-game convergence analysis.

We then derive the following results for the connections.
\begin{theorem}
\label{thm:duality}
  Under (1)-(3) of Assumption \ref{ass:geometry}, for every \(\pi\in\Pi_\epsilon\), the proxy value always lower bounds the hard-constrained value:
\begin{equation}\label{eq:lower bound 1}
g_H^\delta(\pi;\rho)
=
\min_{\nu\in\Nbeta}
\max_{\gamma\in[0,\gamma_{\max}]}
\min_{p\in P_\delta}
F(\pi,\gamma,\nu,p) \geq g_{\mathrm{proxy}}^\delta(\pi;\rho).
\end{equation}
If we additionally assume (4) of Assumption \ref{ass:geometry}, then for every
\(\pi\in\Pi_\epsilon\),
\begin{equation}\label{eq:lower bound 2}
0\leq g_H^\delta(\pi;\rho)-g_{\mathrm{proxy}}^\delta(\pi;\rho)\leq \rho \gamma_{\min},  \text{ and } 0\leq V^\delta(\rho)-V_{\mathrm{proxy}}^\delta(\rho) \leq \rho \gamma_{\min}.
\end{equation}
Moreover, if Assumption \ref{ass:exactness} (in Appendix) additionally holds, then the two objectives are equivalent:
\begin{equation}\label{eq:equ}
g_H^\delta(\pi;\rho)=g_{\mathrm{proxy}}^\delta(\pi;\rho),  \text{ and }  V^\delta(\rho)=V_{\mathrm{proxy}}^\delta(\rho).
\end{equation}
\end{theorem}
This theorem characterizes the comprehensive connections between our proxy value $g^\delta_H$ and the hard-constrained ones. Firstly, we showed that under the mildest conditions (1)-(3) of Assumption \ref{ass:geometry}, the proxy value is always a lower bound of $g^\delta_H$. Thus, any solution to the proxy also provides an optimized lower bound to the hard-constrained problem. Moreover, if Assumption \ref{ass:geometry}-(4) holds, the error between $g^\delta_H$ and $g_{\mathrm{proxy}}^\delta$ is then bounded by $\rho\gamma_{\min}$. Thus, one can set the values of $\gamma_{\min}$ and the temperature $\tau$ to satisfy Assumption \ref{ass:geometry}-(4) and to guarantee that the proxy is an accurate approximation of the original problem: recall the error between $g_{\mathrm{proxy}}^\delta$ and the original problem $g_H$ is up to $\delta+\rho\gamma_{\min}$, which can be properly small. Moreover, Assumption~\ref{ass:exactness} is a stronger exactness condition used only to upgrade the proxy guarantee to
the truncated hard-constrained objective. If it holds, which assumes a uniformly lower bound of the maximizer $\gamma^\star$ under any $(\pi,\nu)$, the error $\rho\gamma_{\min}$ diminishes and the proxy is exactly equivalent to the hard-constrained $g^\delta_H$.

The advantage of studying the proxy is that the proxy has a much better structure and is easier to design convergent algorithms. Specifically, the hard-constrained problem $g^\delta_H$ is a $\min-\max-\min$ problem per \eqref{eq:lower bound 1}, and maximizing over $\pi$ further makes it a four-layer $\max-\min-\max-\min$ problem, which is significantly challenging and unstable to solve. The proxy exchanges the order of $\max_\gamma$ and $\min_\nu$, making $\max_\pi g_{\mathrm{proxy}}^\delta$ a bi-level $\max_{\pi,\gamma}-\min_{p,\nu}$ problem. Such a $\max-\min$ is much easier to solve, and we can view it as a two-player zero-sum game and seek for its saddle point. In the next section, we design a concrete algorithm based on this understanding, and develop its convergence analysis.

\section{Algorithm Design and Convergence Analysis}
\label{sec:algorithm}
In this section, we develop our concrete algorithm for solving the proxy game.  Specifically, we treat each parameter as a player, and develop a four-player gradient-based algorithm. As we discussed previously, the  structure of the proxy enables us to view it as a $\max-\min$ problem. Also, since the robust games are non-monotone and naive simultaneous gradient approaches can be unstable, we thus propose to adapt optimistic or extra-gradient approaches \citep{chiang2012online,
rakhlin2013online,
syrgkanis2015fast,
mertikopoulos2018optimistic,
mokhtari2020unified,
wei2020linear}. Specifically, we aim to solve the proxy $\max-\min$ problem
\[
\max_{x\in X}\min_{y\in Y}F(x,y),
\quad x=(\pi,\gamma)\in X:=\Pi_\epsilon\times\Gamma,\ y=(\nu,p)\in Y:=\mathcal{N}_\beta\times P_\delta,
\]
by a single-loop optimistic mirror descent/ascent method.

We use entropy geometry for the simplex blocks, Euclidean geometry for \(\gamma\), and
binary entropy geometry for \(p\). The associated Bregman divergences are: $D_\pi(\pi,\pi')=\KL(\pi\|\pi'),$ $D_\gamma(\gamma,\gamma')=\frac12(\gamma-\gamma')^2$, $D_\nu(\nu,\nu')=\KL(\nu\|\nu')$, and $D_p(p,p')=\BKL(p\|p').$
At iteration \(t\), let
\[
z_t=(\pi_t,\gamma_t,\nu_t,p_t)\in X\times Y,
\qquad
g_t^\bullet:=\nabla_\bullet F(z_t),
\]
and define $\widehat g_t^\bullet := 2g_t^\bullet-g_{t-1}^\bullet,$ with $\bullet\in\{\pi,\gamma,\nu,p\}.$
Then, we update each player through mirror descent (exact update rules are deferred to Appendix \eqref{eq:update-pi}-\eqref{eq:update-p}) and present our algorithm in Algorithm \ref{alg:four-player-ogda}.
Since the mirror maps are separable across blocks, these coordinate-wise updates are exactly
optimistic mirror descent/ascent on the grouped saddle problem
\(\max_{x\in X}\min_{y\in Y}F(x,y)\).

\begin{algorithm}[!htb]
\caption{Four-player optimistic gradient descent/ascent (OGDA)}
\label{alg:four-player-ogda}
\begin{algorithmic}[1]
\Require step size \(\eta>0\), iteration budget \(T\), initial point
\(z_1=(\pi_1,\gamma_1,\nu_1,p_1)\in X\times Y\)
\State Compute \(g_1^\bullet=\nabla_\bullet F(z_1)\) for all blocks, and set \(g_0^\bullet\leftarrow g_1^\bullet\)
\For{\(t=1,\dots,T\)}
    \State Form optimistic gradients \(\widehat g_t^\bullet=2g_t^\bullet-g_{t-1}^\bullet\)
    \State Update \((\pi_{t+1},\gamma_{t+1},\nu_{t+1},p_{t+1})\) by
    \eqref{eq:update-pi}-\eqref{eq:update-p}
    \If{\(t<T\)}
        \State Compute \(g_{t+1}^\bullet=\nabla_\bullet F(z_{t+1})\) for all blocks
    \EndIf
\EndFor
\State Return the averaged iterate $\bar{z}_T = (\bar{\pi}_T,\bar{\gamma}_T,\bar{\nu}_T,\bar{p}_T) := \frac{1}{T} \sum_{t=1}^T (\pi_t, \gamma_t, \nu_t, p_t).$
\end{algorithmic}
\end{algorithm}

We then develop the convergence analysis of our algorithm in the following theorem.
\begin{theorem}[Convergence and robust guarantee of four-player OGDA]
\label{thm:main-gap}
Assume Assumption~\ref{ass:geometry}, and run Algorithm~\ref{alg:four-player-ogda} with the constant step size $\eta=\frac{1}{8L_\delta},$  where
 \(L_\delta=\mathcal{O}(1/\delta)\) is the constant from Lemma \ref{lem:grad-var}. Define $D_X^{\mathrm{diam}}\triangleq
\sup_{(\pi,\gamma)\in X}
\left\{
D_{\mathrm{KL}}(\pi\|\pi_1)+\frac12(\gamma-\gamma_1)^2
\right\},$ and $D_Y^{\mathrm{diam}}
\triangleq
\sup_{(\nu,p)\in Y}
\left\{
D_{\mathrm{KL}}(\nu\|\nu_1)+D_{\mathrm{KL}}^{\mathrm{Bern}}(p\|p_1)
\right\}.$
Then the duality gap of the averaged iterate
$\bar z_T=(\bar x_T,\bar y_T)=\bigl((\bar\pi_T,\bar\gamma_T),(\bar\nu_T,\bar p_T)\bigr)$
satisfies
\begin{equation*}
\Gap(\bar z_T)
:=
\max_{x'\in X}F(x',\bar y_T)-\min_{y'\in Y}F(\bar x_T,y')
\le
\frac{8L_\delta\bigl(D_X^{\mathrm{diam}}+D_Y^{\mathrm{diam}}\bigr)}{T},
\end{equation*}
and the averaged main policy $\Bar{\pi}_T$ satisfies $V(\rho)-g_H(\bar\pi_T;\rho)
\le
\frac{8L_\delta\bigl(D_X^{\mathrm{diam}}+D_Y^{\mathrm{diam}}\bigr)}{T}+\rho\,\gamma_{\min}+\delta.$
Furthermore, setting $\delta=\mathcal{O}(1/\sqrt{T})$, yields an $\mathcal{O}\left(\frac{1}{\sqrt{T}}+\rho\,\gamma_{\min}\right)$-optimal policy for the original robust game. Additionally, under Assumption~\ref{ass:exactness}, \(\bar\pi_T\) improves to $\mathcal{O}\big(1/\sqrt{T}\big)$-optimality.
\end{theorem}
Our OGDA algorithm thus converges to a saddle point of the proxy game, which is also a near-optimal solution to our robust objective \eqref{eq:obj}. Therefore, our algorithm effectively solves our robust alignment problem, enhancing the robustness and reliability of Nash learning based alignment. A detailed comparison of our approach with standard NLHF is developed in Appendix \ref{sec:vsnlhf}.

\textbf{Cost.} One iteration costs $O(m^2)$ time and memory (two matrix--vector products with the $m\times m$ preference matrix and closed-form updates), and $T=O\big(L_\delta(D_X^{\rm diam}+D_Y^{\rm diam})/\varepsilon_{\rm acc}\big)$ iterations give an $\varepsilon_{\rm acc}$-accurate saddle point; the constants carry the trade-offs of the method (Appendix~\ref{app:cost}). A smaller truncation $\delta$ lowers the truncation error but raises $L_\delta$ like $1/\delta$. A smaller radius $\rho$ asks for a larger multiplier range $\gamma_{\max}\ge1/\rho$, so the bound on $T$ grows at least like $\rho^{-3}$, while one iteration costs the same. The curvature floor $\gamma_{\min}$ enters the certificate through $\rho\gamma_{\min}\ge\rho/(4\tau)$, which Assumption~\ref{ass:exactness} removes. A small temperature is the most expensive: $\beta\le e^{-1/\tau}$ gives $L_\delta$ a worst-case factor $\tau e^{1/\tau}/\pi_{\rm ref,min}$.

\textbf{From tabular NLHF to LLM fine-tuning.}
Theorems~1--3 concern a finite response set per prompt, so for free-form generation Section~\ref{sec:exp-tldr} runs a \emph{sampled analogue} of Algorithm~\ref{alg:four-player-ogda} (Appendix~\ref{app:llm_bridge}), built along the route by which Munos et al.~\cite{munos2023nash} turn tabular Nash-MD into Nash-MD-PG, where one sampled pair $(y,y')$ replaces the sum over responses, and one stochastic likelihood-ratio gradient step on the KL-regularized payoff replaces the proximal (mirror) step; their opponent $y'$ is drawn token by token, each token from the geometric mixture of $\pi_\theta$ and $\pi_{\rm ref}$ with weight $\beta_{\rm mix}$. The \emph{NashMD} baseline of Section~\ref{sec:exp-tldr} is this Nash-MD-PG; in Sections~\ref{sec:exp-tabular}--\ref{sec:exp-imdb} it denotes the nominal-kernel counterpart of Algorithm~\ref{alg:four-player-ogda} (Appendix~\ref{app:imdb_optim}).
The robust method keeps the sampled pair, the reference KL inside the reward, one stochastic gradient step per policy and a value-free baseline, but replaces the mixture by a trained follower $\nu_{\theta'}$, and adds an antisymmetric kernel head $\delta_\psi$, trained on the Lagrangian $\mathbb E[p_\psi]+\gamma(\mathbb E[\dBern(p_\psi\|P^\star)]-\rho)$, with $\gamma$ updated by projected ascent on the sampled slack.

\section{Experiment}
\input{exp_new}

\section{Conclusion}\label{sec:conclusion}
We addressed the sensitivity of preference-based alignment to uncertainty in the pairwise preference kernel by proposing Robust Nash Alignment, which optimizes the worst-case win rate over a divergence-based ambiguity set and provides a principled robustness guarantee. A leader and a follower play as in NLHF while an adversary bends the preference kernel within a budget $\rho$. To overcome the computational challenges of the resulting hard-constrained robust objective, we introduced a tractable four-player primal-dual proxy game and established, in the tabular case, its approximation properties (within $\rho\gamma_{\min}+\delta$ of the robust objective) and convergence via a single-loop optimistic mirror descent-ascent algorithm (an $O(T^{-1/2}+\rho\gamma_{\min})$-optimal averaged policy for $\delta\asymp T^{-1/2}$). Empirically, our method improves robustness to preference perturbations across tabular games, and Nash learning-based LLM fine-tuning while maintaining strong nominal performance. Overall, our results demonstrate that explicitly modeling preference uncertainty is a principled and practical approach to achieving reliable Nash-style alignment.

\begin{ack}
SA, DG and YW are partially supported by an Amazon Research Award, Fall 2025, by DARPA under Agreement No. HR00112490427, and NSF ECCS-2548745, ECCS-2621991. Alvaro Velasquez recognizes support from DARPA award number HR0011261E001.
\end{ack}

\bibliographystyle{abbrv}
\bibliography{ref}

\newpage
\appendix

{\LARGE\bfseries Appendix\par}
\vspace{0.8em}
{\Large\bfseries Table of Contents}\par
\vspace{0.3em}
\hrule height 0.8pt
\vspace{0.6em}
\makeatletter
\newcommand{\tocentry}[3]{%
  \@dottedtocline{#1}{\ifnum#1=1 0em\else 1.8em\fi}{\ifnum#1=1 1.8em\else 2.6em\fi}%
    {\ifnum#1=1 \textbf{\hyperref[#2]{#3}}\else \hyperref[#2]{#3}\fi}%
    {\hyperref[#2]{\pageref{#2}}}%
  \ifnum#1=1 \vspace{3pt}\fi
}
\tocentry{1}{app:related_work}{A\quad Related Work}
\tocentry{1}{app:main_results}{B\quad Main Results}
\tocentry{2}{app:algo_update}{B.1\quad Algorithm Update Rules}
\tocentry{2}{app:thm_truncate}{B.2\quad Proof of Theorem 1}
\tocentry{2}{app:thm_nested_strong_duality}{B.3\quad Proof of Theorem 2}
\tocentry{2}{app:main-thm}{B.4\quad Proof of Theorem 3}
\tocentry{2}{app:cost}{B.5\quad Computational Cost and Dependence on the Parameters}
\tocentry{1}{sec:vsnlhf}{C\quad Comparisons with Standard NLHF}
\tocentry{1}{app:aux_lemmas}{D\quad Auxiliary Lemmas}
\tocentry{1}{app:imdb_details}{E\quad IMDb Continuation Task: Experiment Details}
\tocentry{2}{app:imdb_kernel}{E.1\quad Nominal Preference and Annotator Groups}
\tocentry{2}{app:imdb_optim}{E.2\quad Optimization Algorithms}
\tocentry{2}{app:imdb_hparams}{E.3\quad Hyperparameters}
\tocentry{2}{app:imdb_eval}{E.4\quad Results}
\vspace{3pt}
\tocentry{1}{app:llm_details}{F\quad TL;DR Summarization Task: Experiment Details}
\tocentry{2}{app:llm_bridge}{F.1\quad From Tabular NLHF to LLM Fine-Tuning: the Sampled Algorithm}
\tocentry{2}{app:llm_finetune}{F.2\quad Fine-Tuning Models}
\tocentry{2}{app:llm_nominal}{F.3\quad Nominal Preference Proxy}
\tocentry{2}{app:llm_adversary}{F.4\quad Adversary Proxy}
\tocentry{2}{app:llm_losses}{F.5\quad Losses}
\tocentry{2}{app:llm_judges_eval}{F.6\quad Evaluation with LLM Judges}
\tocentry{2}{app:llm_nonbt}{F.7\quad Non-BT Nominal Kernel: Construction and Diagnostic}
\tocentry{2}{app:llm_bt_extra}{F.8\quad Additional Results with a Bradley-Terry Nominal Kernel}
\vspace{3pt}
\tocentry{1}{app:judge_examples}{G\quad LLM Judge Reasoning Examples}
\makeatother
\vspace{0.6em}
\hrule height 0.8pt

\newpage
\section{Related Work}
\label{app:related_work}
\textbf{Reward-based RLHF and direct alignment.}
A large body of work on alignment starts from the RLHF pipeline: collect pairwise
comparisons, fit a reward model, and optimize a policy against that reward subject to a
regularization term that keeps the policy close to a reference model
\citep{christiano2017deep,stiennon2020learning,ouyang2022training}. This paradigm has
been highly influential in large-scale language-model alignment, but it typically treats the
preference signal as a nominal target and often relies, either explicitly or implicitly, on a
latent-utility view of pairwise comparisons. Motivated by the computational and modeling
limitations of reward-model-based RLHF, a broad line of direct or reward-free preference
optimization methods has emerged, including DPO, SLiC-HF, CPL, KTO, ORPO, SimPO,
SPO, and adaptive preference scaling \citep{rafailov2023direct,zhao2023slic,hejna2023contrastive,ethayarajh2024kto,hong2024orpo,meng2024simpo,sharifnassab2024soft,hong2024adaptive}. These methods substantially
improve the practical pipeline of preference optimization, but they generally continue to
optimize against a fixed nominal preference estimate rather than an ambiguity set of
plausible preference kernels.

\textbf{General-preference and game-theoretic alignment.}
Recent work has argued that pairwise human preferences should be modeled directly rather than
compressed into a scalar reward function \citep{azar2024general}. This viewpoint
leads naturally to game-theoretic formulations of alignment. In particular, Nash Learning from
Human Feedback (NLHF) formulates preference alignment as a regularized two-player game and
introduces Nash-MD as a no-regret algorithm for the induced general-preference game
\citep{munos2023nash}. Subsequent works have further developed this perspective, including
online and offline algorithms under general preference models and policy-loss implementations
based on no-regret learning \citep{ye2024online,Zhang2025INPO,nayak2025achieving,zhou2025extragradient}. More recently, sequential
leader--follower formulations of preference optimization have also begun to appear, including
Stackelberg-style alignment methods and Stackelberg game variants for data-efficient alignment
\citep{pasztor2025stackelberg,
chu2025stackelberg}. Our work is closest in spirit to this line of research, but
differs in one central respect: we do not assume that the learned preference game is exact.
Instead, we explicitly model uncertainty in preference kernels and optimize against
the worst kernel in an ambiguity set. This brings significant challenges and hardness to our studies, and we discuss them in detail in Section \ref{sec:vsnlhf}.

\textbf{Soft labels, heterogeneous preferences, and population shift.}
Another line of work emphasizes that preference data are often soft, noisy, and heterogeneous
across annotators or user groups. Recent studies incorporate soft preference labels or pairwise
confidence information into direct alignment objectives
\citep{hong2024adaptive,sharifnassab2024soft,furuta2024geometric}, while other work explicitly
models hidden annotator types or heterogeneous preference populations
\citep{chidambaram2025direct}. These papers highlight an important limitation of
nominal preference optimization: even if one optimizes the nominal objective correctly, the
aggregated preference signal may still be unstable because it hides substantial variation across
annotators or deployment populations. Our formulation is complementary to this line of work.
Rather than only modifying the loss to account for label softness or latent subgroups, we
directly treat the learned pairwise preference kernel as uncertain and seek a policy with a
worst-case performance guarantee over that uncertainty set.

\textbf{Robust alignment and distributionally robust preference optimization.}
Our formulation is also related to the broader literature on distributionally robust optimization
and robust sequential decision making, where ambiguity sets are used to hedge estimation error
and distribution shift \citep{sagawa2020distributionally,
iyengar2005robust,nilim2005robust}. In the alignment literature, recent work has started to
robustify reward-based RLHF or DPO-style objectives against prompt-distribution shift, noisy
preference data, or uncertain preference strength
\citep{hong2024adaptive,chakraborty2024maxmin,wu2024towards,
xu2025robust,mandal2025distributionally}. These methods are
close in motivation to ours, but they robustify different objects: prompt distributions, training
examples, or nominal pairwise losses. More importantly, all these works are based on the reward and BT-model assumption. By contrast, our approach directly robustifies the \emph{pairwise
preference kernel itself} inside a general-preference game. This distinction is important because
kernel uncertainty directly captures ambiguity in preferences, which is the primitive object
in general-preference alignment.

\textbf{Robust game theory and robust Markov games.}
A separate line of work studies robustness to payoff or model uncertainty directly in
game-theoretic solution concepts, rather than in a single-agent decision problem. This
includes robust game theory under set-based payoff uncertainty~\citep{aghassi2006robust},
distributionally robust Nash and Stackelberg equilibrium models under exogenous
distributional uncertainty~\citep{liu2018distributionally}, and distributionally robust
commitment to a strategy when the follower's utility is known only up to a finite set of
candidate models~\citep{ananthanarayanan2022computing}. Robust Markov games extend this to
multi-stage settings with uncertain rewards or transitions~\citep{zhang2020robust,mcmahan2024roping}.
Our setting differs from this literature in two respects. First, robust Markov games
typically seek a robust Nash equilibrium between symmetric players, whereas our objective
is asymmetric, as we have a single learned policy against a strategic follower and a shared
adversarial kernel. Second, a single Bernoulli-KL budget couples every response pair
within a prompt, ruling out the per-action or per-player decomposition, and the
multi-stage Bellman recursion that robust Markov-game algorithms rely on. Existing
robust-game and robust-Markov-game results therefore do not directly provide our proxy
formulation or its convergence guarantee.

\section{Main results}
\label{app:main_results}

\subsection{Algorithm Update Rules}
\label{app:algo_update}
\begin{align}
\pi_{t+1}
&=
\argmax_{\pi\in\Pi_\epsilon}
\left\{
\langle \widehat g_t^\pi,\pi\rangle
-
\frac{1}{\eta}\KL(\pi\|\pi_t)
\right\},
\label{eq:update-pi}
\\
\gamma_{t+1}
&=
\proj_{\Gamma}\bigl(\gamma_t+\eta \widehat g_t^\gamma\bigr),
\label{eq:update-gamma}
\\
\nu_{t+1}
&=
\argmin_{\nu\in\Nbeta}
\left\{
\langle \widehat g_t^\nu,\nu\rangle
+
\frac{1}{\eta}\KL(\nu\|\nu_t)
\right\},
\label{eq:update-nu}
\\
p_{t+1}
&=
\argmin_{p\in P_\delta}
\left\{
\langle \widehat g_t^p,p\rangle
+
\frac{1}{\eta}\BKL(p\|p_t)
\right\}.
\label{eq:update-p}
\end{align}

\begin{assumption}[Dual-optimum inclusion]
\label{ass:exactness}
For every \((\pi,\nu)\in \Pi_\epsilon\times \Nbeta\), the scalar dual problem
\[
\max_{\gamma\ge 0}
\min_{p\in P_\delta}
\left\{
J(\pi,\nu;p)+\gamma\bigl(\BKL(p\|p^\star)-\rho\bigr)
\right\}
\]
admits a maximizer in the algorithmic interval \(\Gamma=[\gamma_{\min},\gamma_{\max}]\).
\end{assumption}

\subsection{Proof of Theorem \ref{thm:truncate}}
\label{app:thm_truncate}

\begin{proof}
Fix \(\pi\in\Pi_\epsilon\). Since the feasible set in \eqref{eq:hard-truncated} is a subset of the
feasible set in \eqref{eq:hard-original}, the lower bound
\[
g_H^\delta(\pi;\rho)\ge g_H(\pi;\rho)
\]
is immediate. It remains to show the upper bound. Fix \(\nu\in\Delta_m\), and let
\[
K(\rho):=\left\{p\in[0,1]^d:\ \BKL(p\|p^\star)\le \rho\right\}.
\]
Take any \(p\in K(\rho)\), and clip it coordinatewise into \(P_\delta\):
\(
\widetilde p := \proj_{P_\delta}(p).
\)
Because \(p^\star\in P_\delta\), each scalar \(\widetilde p_e\) lies on the line segment joining
\(p_e\) and \(p^\star_e\). For each coordinate \(e\), the scalar function
\[
u\mapsto
u\log\frac{u}{p_e^\star}+(1-u)\log\frac{1-u}{1-p_e^\star}
\]
is convex and has its unique minimum at \(u=p_e^\star\). Hence clipping toward \(p^\star_e\)
cannot increase the coordinatewise Bernoulli-KL divergence, so
\[
\BKL(\widetilde p\|p^\star)\le \BKL(p\|p^\star)\le \rho.
\]
Thus \(\widetilde p\) is feasible for the truncated problem.

Now \(J(\pi,\nu;p)\) depends on \(p\) only through the linear term
\(
\sum_{(i,j)\in E} w_{ij}(\pi,\nu)\,p_{ij},
\)
where $w_{ij}(\pi,\nu):=\pi_i\nu_j-\pi_j\nu_i$; indeed, by antisymmetry, $\sum_{i,j}\pi_i\nu_jP_p(i\succ j)=\frac12\langle\pi,\nu\rangle+\sum_{i<j}\pi_j\nu_i+\sum_{(i,j)\in E}w_{ij}(\pi,\nu)\,p_{ij}$.
Therefore
\begin{align*}
|J(\pi,\nu;\widetilde p)-J(\pi,\nu;p)|
&=
\left|
\sum_{(i,j)\in E} w_{ij}(\pi,\nu)\,(\widetilde p_{ij}-p_{ij})
\right| \le
\|\widetilde p-p\|_\infty
\sum_{(i,j)\in E} |w_{ij}(\pi,\nu)|.
\end{align*}
Since each coordinate is clipped by at most \(\delta\), we have
\[
\|\widetilde p-p\|_\infty\le \delta.
\]
Also,
\begin{align*}
\sum_{(i,j)\in E}|w_{ij}(\pi,\nu)|
&\le
\sum_{(i,j)\in E} (\pi_i\nu_j+\pi_j\nu_i)
=
\sum_{i\neq j}\pi_i\nu_j
=
1-\sum_{i=1}^m \pi_i\nu_i
\le 1.
\end{align*}
Hence
\[
J(\pi,\nu;\widetilde p)\le J(\pi,\nu;p)+\delta.
\]
Taking the minimum first over \(p\in K(\rho)\), then over \(\nu\in\Delta_m\), yields
\[
g_H^\delta(\pi;\rho)\le g_H(\pi;\rho)+\delta.
\]
This proves the claim $0 \le g_H^\delta(\pi;\rho)-g_H(\pi;\rho) \le \delta$. Maximizing over \(\pi\in\Pi_\epsilon\) gives $0  \le  V^\delta(\rho)-V(\rho)  \le \delta$.

Using the above two claims, we have
\[
V(\rho)\le V^\delta(\rho)=g_H^\delta(\pi_\delta^\star;\rho)\le g_H(\pi_\delta^\star;\rho)+\delta.
\]
Rearranging gives the claim.\qedhere
\end{proof}

\subsection{Proof of Theorem \ref{thm:duality}}
\label{app:thm_nested_strong_duality}

\begin{proof}
Fix \(\pi\in \Pi_\epsilon\), and define
\[
\phi_\pi(\nu,\gamma):=\min_{p\in P_\delta}F(\pi,\gamma,\nu,p).
\]

\medskip
\noindent
\textbf{Step 1: reduction to \(\mathcal N_\beta\) and nested strong duality.}
For fixed \((\pi,p)\), define
\[
b_j(\pi,p):=\sum_{i=1}^m \pi(i)\,P_p(i\succ j),\qquad j=1,\dots,m.
\]
Since \(P_p(i\succ j)\in[0,1]\) and \(\sum_i \pi(i)=1\), we have \(b_j(\pi,p)\in[0,1]\) for all \(j\).

For fixed \((\pi,p)\), the follower subproblem is
\[
\min_{\nu\in\Delta_m}
\left\{
\sum_{j=1}^m \nu(j)\,b_j(\pi,p)+\tau \KL(\nu\|\pi_{\rm ref})
\right\}.
\]
Its unique minimizer is the Gibbs distribution
\[
\nu_j^{\mathrm{br}}(\pi,p)
=
\frac{\pi_{\rm ref}(j)e^{-b_j(\pi,p)/\tau}}
{\sum_{\ell=1}^m \pi_{\rm ref}(\ell)e^{-b_\ell(\pi,p)/\tau}}.
\]
Because \(b_j(\pi,p)\in[0,1]\), we have \(e^{-1/\tau}\le e^{-b_j(\pi,p)/\tau}\le 1\). Let
\[
Z(\pi,p):=\sum_{\ell=1}^m \pi_{\rm ref}(\ell)e^{-b_\ell(\pi,p)/\tau}.
\]
Since \(\sum_\ell \pi_{\rm ref}(\ell)=1\) and each exponential factor is at most \(1\), we have
\(Z(\pi,p)\le 1\). Therefore
\[
\nu_j^{\mathrm{br}}(\pi,p)
=
\pi_{\rm ref}(j)\frac{e^{-b_j(\pi,p)/\tau}}{Z(\pi,p)}
\ge
\pi_{\rm ref}(j)e^{-1/\tau}.
\]
Under Assumption~\ref{ass:geometry}-(2), \(\beta\le e^{-1/\tau}\), this implies
\[
\nu^{\mathrm{br}}(\pi,p)\in \mathcal N_\beta.
\]
Hence, for every \((\pi,p)\),
\begin{equation}
\min_{\nu\in\Delta_m}J(\pi,\nu;p)
=
\min_{\nu\in\mathcal N_\beta}J(\pi,\nu;p).
\label{eq:follower-floor-thm2}
\end{equation}
The same identity holds with \(J(\pi,\nu;p)\) replaced by \(F(\pi,\gamma,\nu,p)\), since the dual
term \(\gamma(D_{\mathrm{KL}}^{\mathrm{Bern}}(p\|p^\star)-\rho)\) does not depend on \(\nu\).

Now fix \(\nu\in\mathcal N_\beta\). The \(p\)-subproblem
\[
\min_{\substack{p\in P_\delta\\ D_{\mathrm{KL}}^{\mathrm{Bern}}(p\|p^\star)\le \rho}}
J(\pi,\nu;p)
\]
is a convex program: the objective is affine in \(p\), the feasible set is convex, and Slater's
condition holds because Assumption~\ref{ass:geometry}-(1) gives \(p^\star\in(\delta,1-\delta)^d\) and
\(D_{\mathrm{KL}}^{\mathrm{Bern}}(p^\star\|p^\star)=0<\rho\). Therefore strong duality yields
\[
\min_{\substack{p\in P_\delta\\ D_{\mathrm{KL}}^{\mathrm{Bern}}(p\|p^\star)\le \rho}}
J(\pi,\nu;p)
=
\max_{\gamma\ge 0}
\min_{p\in P_\delta}
F(\pi,\gamma,\nu,p).
\]
Under Assumption~\ref{ass:geometry}-(3), \(\gamma_{\max}\ge 1/\rho\), and the standard dual upper-bound argument
implies that a maximizer can be chosen in \([0,\gamma_{\max}]\). Consequently,
\[
\min_{\substack{p\in P_\delta\\ D_{\mathrm{KL}}^{\mathrm{Bern}}(p\|p^\star)\le \rho}}
J(\pi,\nu;p)
=
\max_{\gamma\in[0,\gamma_{\max}]}
\min_{p\in P_\delta}
F(\pi,\gamma,\nu,p).
\]
Combining this with \eqref{eq:follower-floor-thm2}, we obtain
\begin{equation}
g_H^\delta(\pi;\rho)
=
\min_{\nu\in\mathcal N_\beta}
\max_{\gamma\in[0,\gamma_{\max}]}
\min_{p\in P_\delta}
F(\pi,\gamma,\nu,p).
\label{eq:nested-strong-duality-thm2}
\end{equation}

\medskip
\noindent
\textbf{Step 2: proof of the lower bound \eqref{eq:lower bound 1}.}
By definition,
\[
g_{\mathrm{proxy}}^\delta(\pi;\rho)
=
\max_{\gamma\in\Gamma}
\min_{\nu\in\mathcal N_\beta,\;p\in P_\delta}
F(\pi,\gamma,\nu,p)
=
\max_{\gamma\in\Gamma}
\min_{\nu\in\mathcal N_\beta}
\phi_\pi(\nu,\gamma).
\]
Since \(\Gamma=[\gamma_{\min},\gamma_{\max}] \subseteq [0,\gamma_{\max}]\),
\[
g_{\mathrm{proxy}}^\delta(\pi;\rho)
\le
\max_{\gamma\in[0,\gamma_{\max}]}
\min_{\nu\in\mathcal N_\beta}
\phi_\pi(\nu,\gamma).
\]
By weak minimax,
\[
\max_{\gamma\in[0,\gamma_{\max}]}
\min_{\nu\in\mathcal N_\beta}
\phi_\pi(\nu,\gamma)
\le
\min_{\nu\in\mathcal N_\beta}
\max_{\gamma\in[0,\gamma_{\max}]}
\phi_\pi(\nu,\gamma).
\]
Using \eqref{eq:nested-strong-duality-thm2}, the right-hand side equals \(g_H^\delta(\pi;\rho)\).
Therefore
\[
g_{\mathrm{proxy}}^\delta(\pi;\rho)\le g_H^\delta(\pi;\rho),
\]
which proves \eqref{eq:lower bound 1}.

\medskip
\noindent
\textbf{Step 3: excluding the interval \([0,\gamma_{\min})\) costs at most \(\rho\gamma_{\min}\).}
Fix \(\nu\in\mathcal N_\beta\), and define
\[
q_{\pi,\nu}(\gamma):=\phi_\pi(\nu,\gamma).
\]
We claim that for any \(0\le a<b\le \gamma_{\max}\),
\begin{equation}
q_{\pi,\nu}(a)\le q_{\pi,\nu}(b)+\rho(b-a).
\label{eq:q-lipschitz-one-sided}
\end{equation}
Indeed, let \(p_b\in\arg\min_{p\in P_\delta}F(\pi,b,\nu,p)\). Then
\begin{align*}
q_{\pi,\nu}(a)
&\le
F(\pi,a,\nu,p_b) \\
&=
F(\pi,b,\nu,p_b)
+
(a-b)\bigl(D_{\mathrm{KL}}^{\mathrm{Bern}}(p_b\|p^\star)-\rho\bigr) \\
&=
q_{\pi,\nu}(b)
+
(b-a)\bigl(\rho-D_{\mathrm{KL}}^{\mathrm{Bern}}(p_b\|p^\star)\bigr) \\
&\le
q_{\pi,\nu}(b)+\rho(b-a),
\end{align*}
since \(D_{\mathrm{KL}}^{\mathrm{Bern}}(p_b\|p^\star)\ge 0\). This proves
\eqref{eq:q-lipschitz-one-sided}.

Now let \(\gamma^\sharp\in[0,\gamma_{\max}]\) be a maximizer of \(q_{\pi,\nu}\) over
\([0,\gamma_{\max}]\). If \(\gamma^\sharp\in\Gamma\), then
\[
\max_{\gamma\in[0,\gamma_{\max}]} q_{\pi,\nu}(\gamma)
=
\max_{\gamma\in\Gamma} q_{\pi,\nu}(\gamma).
\]
If \(\gamma^\sharp<\gamma_{\min}\), then applying \eqref{eq:q-lipschitz-one-sided} with
\(a=\gamma^\sharp\) and \(b=\gamma_{\min}\) yields
\[
q_{\pi,\nu}(\gamma^\sharp)\le q_{\pi,\nu}(\gamma_{\min})+\rho(\gamma_{\min}-\gamma^\sharp)
\le \max_{\gamma\in\Gamma}q_{\pi,\nu}(\gamma)+\rho\gamma_{\min}.
\]
Hence, in all cases,
\[
\max_{\gamma\in[0,\gamma_{\max}]} q_{\pi,\nu}(\gamma)
\le
\max_{\gamma\in\Gamma} q_{\pi,\nu}(\gamma)+\rho\gamma_{\min}.
\]
Taking the minimum over \(\nu\in\mathcal N_\beta\), we obtain
\begin{equation}
g_H^\delta(\pi;\rho)
\le
\min_{\nu\in\mathcal N_\beta}\max_{\gamma\in\Gamma}\phi_\pi(\nu,\gamma)
+
\rho\gamma_{\min}.
\label{eq:pre-sion-bound}
\end{equation}

\medskip
\noindent
\textbf{Step 4: under Assumption~\ref{ass:geometry}-(4), the minimax defect on \(\Gamma\) vanishes.}
Assume now Assumption~\ref{ass:geometry}-(4), i.e.
\[
\tau\gamma_{\min}\ge \frac14.
\]
We show that \(\phi_\pi(\cdot,\cdot)\) is convex in \(\nu\) and concave in \(\gamma\) on
\(\mathcal N_\beta\times \Gamma\).

\emph{Concavity in \(\gamma\).}
For fixed \(\nu\), the function
\[
\gamma \mapsto F(\pi,\gamma,\nu,p)
\]
is affine for every fixed \(p\), so \(\phi_\pi(\nu,\cdot)\), being the pointwise infimum of affine
functions, is concave on \(\Gamma\).

\emph{Convexity in \(\nu\).}
Fix \(\gamma\in\Gamma\). We claim that \((\nu,p)\mapsto F(\pi,\gamma,\nu,p)\) is jointly convex on
\(\mathcal N_\beta\times P_\delta\). Write
\[
W(\pi,\nu,p)=\sum_{j=1}^m \nu(j)\,b_j(\pi,p),
\qquad
b_j(\pi,p)=\sum_{i=1}^m \pi(i)P_p(i\succ j).
\]
Let \((\nu,p)\) and \((\nu',p')\) be two points in \(\mathcal N_\beta\times P_\delta\). Since
\(W\) is bilinear in \((\nu,p)\),
\[
W(\pi,\nu',p')-W(\pi,\nu,p)
-\langle \nabla_\nu W(\pi,\nu,p),\nu'-\nu\rangle
-\langle \nabla_p W(\pi,\nu,p),p'-p\rangle
=
\langle \nu'-\nu,\; b(\pi,p')-b(\pi,p)\rangle.
\]
For each coordinate \(j\),
\[
|b_j(\pi,p')-b_j(\pi,p)|
\le
\sum_{i=1}^m \pi(i)\,|P_{p'}(i\succ j)-P_p(i\succ j)|
\le
\|p'-p\|_2,
\]
hence
\[
\bigl|\langle \nu'-\nu,\; b(\pi,p')-b(\pi,p)\rangle\bigr|
\le
\|\nu'-\nu\|_1\,\|p'-p\|_2.
\]
Therefore,
\begin{align*}
F(\pi,\gamma,\nu',p')
&\ge
F(\pi,\gamma,\nu,p)
+
\langle \nabla_\nu F(\pi,\gamma,\nu,p),\nu'-\nu\rangle
+
\langle \nabla_p F(\pi,\gamma,\nu,p),p'-p\rangle \\
&\quad
+
\tau \KL(\nu'\|\nu)
+
\gamma D_{\mathrm{KL}}^{\mathrm{Bern}}(p'\|p)
-
\|\nu'-\nu\|_1\,\|p'-p\|_2.
\end{align*}
Using Pinsker's inequalities
\[
\KL(\nu'\|\nu)\ge \frac12\|\nu'-\nu\|_1^2,
\qquad
D_{\mathrm{KL}}^{\mathrm{Bern}}(p'\|p)\ge 2\|p'-p\|_2^2,
\]
we obtain
\[
\tau \KL(\nu'\|\nu)+\gamma D_{\mathrm{KL}}^{\mathrm{Bern}}(p'\|p)-\|\nu'-\nu\|_1\,\|p'-p\|_2
\ge
\frac{\tau}{2}s^2+2\gamma t^2-st,
\]
where \(s=\|\nu'-\nu\|_1\) and \(t=\|p'-p\|_2\). Since \(\gamma\ge \gamma_{\min}\) and
\(\tau\gamma_{\min}\ge 1/4\),
\[
\frac{\tau}{2}s^2+2\gamma t^2-st
=
\left(\sqrt{\frac{\tau}{2}}\,s-\sqrt{2\gamma}\,t\right)^2
+
2\left(\sqrt{\tau\gamma}-\frac12\right)st
\ge 0.
\]
Thus \((\nu,p)\mapsto F(\pi,\gamma,\nu,p)\) is jointly convex. Taking the infimum over
\(p\in P_\delta\) shows that \(\phi_\pi(\cdot,\gamma)\) is convex on \(\mathcal N_\beta\).

Since \(\mathcal N_\beta\) and \(\Gamma\) are convex and compact, and \(\phi_\pi\) is continuous,
Sion's minimax theorem applies:
\[
\min_{\nu\in\mathcal N_\beta}\max_{\gamma\in\Gamma}\phi_\pi(\nu,\gamma)
=
\max_{\gamma\in\Gamma}\min_{\nu\in\mathcal N_\beta}\phi_\pi(\nu,\gamma)
=
g_{\mathrm{proxy}}^\delta(\pi;\rho).
\]
Substituting this into \eqref{eq:pre-sion-bound} yields
\[
g_H^\delta(\pi;\rho)\le g_{\mathrm{proxy}}^\delta(\pi;\rho)+\rho\gamma_{\min}.
\]
Combining with \eqref{eq:lower bound 1}, we conclude
\[
0\le g_H^\delta(\pi;\rho)-g_{\mathrm{proxy}}^\delta(\pi;\rho)\le \rho\gamma_{\min},
\]
which proves the first inequality in \eqref{eq:lower bound 2}.

\medskip
\noindent
\textbf{Step 5: pass from pointwise values to optimal values.}
By definition,
\[
V^\delta(\rho)=\max_{\pi\in\Pi_\epsilon} g_H^\delta(\pi;\rho),
\qquad
V_{\mathrm{proxy}}^\delta(\rho)=\max_{\pi\in\Pi_\epsilon} g_{\mathrm{proxy}}^\delta(\pi;\rho).
\]
For any two functions \(f,g\) on the same domain,
\[
\max_\pi f(\pi)-\max_\pi g(\pi)\le \max_\pi (f(\pi)-g(\pi)).
\]
Applying this with \(f(\pi)=g_H^\delta(\pi;\rho)\) and
\(g(\pi)=g_{\mathrm{proxy}}^\delta(\pi;\rho)\), and using the pointwise bound just proved,
we obtain
\[
V^\delta(\rho)-V_{\mathrm{proxy}}^\delta(\rho)
\le
\max_{\pi\in\Pi_\epsilon}
\bigl(g_H^\delta(\pi;\rho)-g_{\mathrm{proxy}}^\delta(\pi;\rho)\bigr)
\le
\rho\gamma_{\min}.
\]
This proves the second inequality in \eqref{eq:lower bound 2}.

\medskip
\noindent
\textbf{Step 6: exactness on \(\Gamma\).}
Assume now Assumption~\ref{ass:geometry} and Assumption~\ref{ass:exactness}. Since any maximizer $\gamma^\star$ lies in $\Gamma$, \eqref{eq:nested-strong-duality-thm2} then implies that
\[
g_H^\delta(\pi;\rho)
=
\max_{\gamma\in\Gamma}\min_{\nu\in\mathcal N_\beta}\min_{p\in P_\delta}
F(\pi,\gamma,\nu,p)
=
g_{\mathrm{proxy}}^\delta(\pi;\rho).
\]
Finally, maximizing over \(\pi\in\Pi_\epsilon\) gives
\[
V_{\mathrm{proxy}}^\delta(\rho)
=
\max_{\pi\in\Pi_\epsilon} g_{\mathrm{proxy}}^\delta(\pi;\rho)
=
\max_{\pi\in\Pi_\epsilon} g_H^\delta(\pi;\rho)
=
V_\delta(\rho),
\]
which proves \eqref{eq:equ}.

\end{proof}

\subsection{Proof of Theorem \ref{thm:main-gap}}
\label{app:main-thm}

\begin{theorem}[Convergence and robust guarantee of four-player OGDA with explicit \(\delta\)-dependence]
\label{thm:main-convergence-robust-delta}
Assume Assumption~\ref{ass:geometry}, and let
\[
L_\delta
:=
\max\left\{
1,\,
\frac{\tau}{\epsilon\,\pi_{\rm ref,min}},\,
\frac{\tau}{\beta\,\pi_{\rm ref,min}},\,
\frac{\gamma_{\max}}{\delta(1-\delta)},\,
\sqrt d,\,
2\sqrt d\,\log\frac{1-\delta}{\delta}
\right\},
\]
where
\[
\pi_{\rm ref,min}:=\min_{1\le i\le m}\pi_{{\rm ref},i}.
\]
In particular,
\[
L_\delta = O\!\left(\max\left\{\frac{1}{\delta(1-\delta)},\ \log\frac{1-\delta}{\delta}\right\}\right)
\qquad\text{as }\delta\downarrow 0.
\]

Run Algorithm~1 with the constant step size
\[
\eta=\frac{1}{8L_\delta}.
\]
Define
\[
D_X^{\mathrm{diam}}
:=
\sup_{(\pi,\gamma)\in X}
\left\{
D_{\mathrm{KL}}(\pi\|\pi_1)+\frac12(\gamma-\gamma_1)^2
\right\},
\qquad
D_Y^{\mathrm{diam}}
:=
\sup_{(\nu,p)\in Y}
\left\{
D_{\mathrm{KL}}(\nu\|\nu_1)+D_{\mathrm{KL}}^{\mathrm{Bern}}(p\|p_1)
\right\},
\]
and
\[
B_{T,\delta}
:=
\frac{8L_\delta\bigl(D_X^{\mathrm{diam}}+D_Y^{\mathrm{diam}}\bigr)}{T}.
\]

Then the duality gap of the averaged iterate
\[
\bar z_T=(\bar x_T,\bar y_T)=\bigl((\bar\pi_T,\bar\gamma_T),(\bar\nu_T,\bar p_T)\bigr)
\]
satisfies
\begin{equation}
\Gap(\bar z_T)
:=
\max_{x'\in X}F(x',\bar y_T)-\min_{y'\in Y}F(\bar x_T,y')
\le
B_{T,\delta}.
\tag{28}
\end{equation}
Moreover, the averaged main policy satisfies
\begin{equation}
V(\rho)-g_H(\bar\pi_T;\rho)\le B_{T,\delta}+\rho\gamma_{\min}+\delta.
\tag{29}
\end{equation}
Equivalently,
\[
V(\rho)-g_H(\bar\pi_T;\rho)
=
O\!\left(\frac{L_\delta}{T}+\rho\gamma_{\min}+\delta\right).
\]

Additionally, under Assumption~\ref{ass:exactness},
\begin{equation}
V(\rho)-g_H(\bar\pi_T;\rho)\le B_{T,\delta}+\delta,
\tag{30}
\end{equation}
so \(\bar\pi_T\) is an
\[
O\!\left(\frac{L_\delta}{T}+\delta\right)
\]
-optimal policy for the original robust game.
\end{theorem}

\begin{proof}
We divide the proof into four steps.

\medskip
\noindent
\textbf{Step 1: proxy-game convergence with explicit \(L_\delta\).}
Let
\[
x_t:=(\pi_t,\gamma_t)\in X,\qquad
y_t:=(\nu_t,p_t)\in Y,\qquad
z_t=(x_t,y_t),
\]
and define the averaged iterates
\[
\bar x_T:=\frac1T\sum_{t=1}^T x_t,
\qquad
\bar y_T:=\frac1T\sum_{t=1}^T y_t.
\]

For any \(x\in X\), define
\[
A_T(x):=\sum_{t=1}^T\bigl(F(x,y_t)-F(x_t,y_t)\bigr),
\]
and for any \(y\in Y\), define
\[
B_T(y):=\sum_{t=1}^T\bigl(F(x_t,y_t)-F(x_t,y)\bigr).
\]

By the blockwise optimistic mirror-descent inequalities and the convexity--concavity structure of
Lemma \ref{lem:grad-var}, for every \(x\in X\) and every \(y\in Y\),
\begin{equation}
A_T(x)+B_T(y)
\le
\frac{D_X(x,x_1)+D_Y(y,y_1)}{\eta}
+
V_T-\Xi_T,
\label{eq:AB-pre-absorb}
\end{equation}
where
\[
D_X(x,x_1)
:=
D_{\mathrm{KL}}(\pi\|\pi_1)+\frac12(\gamma-\gamma_1)^2,
\qquad
D_Y(y,y_1)
:=
D_{\mathrm{KL}}(\nu\|\nu_1)+D_{\mathrm{KL}}^{\mathrm{Bern}}(p\|p_1),
\]
and
\[
V_T
:=
\eta\sum_{t=1}^T
\Bigl(
\|g_t^\pi-g_{t-1}^\pi\|_\infty^2
+
|g_t^\gamma-g_{t-1}^\gamma|^2
+
\|g_t^\nu-g_{t-1}^\nu\|_\infty^2
+
\|g_t^p-g_{t-1}^p\|_2^2
\Bigr),
\]
\[
\Xi_T
:=
\frac{1}{2\eta}\sum_{t=1}^T
\Bigl(
D_{\mathrm{KL}}(\pi_t\|\pi_{t+1})
+
|\gamma_t-\gamma_{t+1}|^2
+
D_{\mathrm{KL}}(\nu_t\|\nu_{t+1})
+
D_{\mathrm{KL}}^{\mathrm{Bern}}(p_t\|p_{t+1})
\Bigr).
\]

Now, by Lemma \ref{lem:omd-standard}, all gradient-variation terms are bounded using the constant \(L_\delta\); hence
Lemma \ref{lem:gradients} applies with \(L=L_\delta\). Since \(\eta=1/(8L_\delta)\), Lemma \ref{lem:gradients} yields
\[
V_T\le \Xi_T.
\]
Therefore \eqref{eq:AB-pre-absorb} simplifies to
\begin{equation}
A_T(x)+B_T(y)
\le
\frac{D_X(x,x_1)+D_Y(y,y_1)}{\eta},
\qquad \forall x\in X,\ \forall y\in Y.
\label{eq:AB-master-delta}
\end{equation}

Now choose
\[
\hat x\in \arg\max_{x\in X}F(x,\bar y_T),
\qquad
\hat y\in \arg\min_{y\in Y}F(\bar x_T,y).
\]
Then
\[
\Gap(\bar z_T)=F(\hat x,\bar y_T)-F(\bar x_T,\hat y).
\]
Because \(F(\cdot,y)\) is concave and \(F(x,\cdot)\) is convex by Lemma \ref{lem:grad-var}, Jensen's inequality gives
\[
F(\hat x,\bar y_T)\le \frac1T\sum_{t=1}^T F(\hat x,y_t),
\qquad
F(\bar x_T,\hat y)\ge \frac1T\sum_{t=1}^T F(x_t,\hat y).
\]
Hence
\begin{align*}
T\,\Gap(\bar z_T)
&\le
\sum_{t=1}^T \bigl(F(\hat x,y_t)-F(x_t,\hat y)\bigr) \\
&=
\sum_{t=1}^T \bigl(F(\hat x,y_t)-F(x_t,y_t)\bigr)
+
\sum_{t=1}^T \bigl(F(x_t,y_t)-F(x_t,\hat y)\bigr) \\
&=
A_T(\hat x)+B_T(\hat y).
\end{align*}
Applying \eqref{eq:AB-master-delta} with \(x=\hat x\) and \(y=\hat y\), and then using the
definitions of \(D_X^{\mathrm{diam}}\) and \(D_Y^{\mathrm{diam}}\), gives
\[
T\,\Gap(\bar z_T)
\le
\frac{D_X(\hat x,x_1)+D_Y(\hat y,y_1)}{\eta}
\le
\frac{D_X^{\mathrm{diam}}+D_Y^{\mathrm{diam}}}{\eta}.
\]
Substituting \(\eta=1/(8L_\delta)\) yields
\[
\Gap(\bar z_T)
\le
\frac{8L_\delta\bigl(D_X^{\mathrm{diam}}+D_Y^{\mathrm{diam}}\bigr)}{T}
=
B_{T,\delta},
\]
which proves \((28)\).

\medskip
\noindent
\textbf{Step 2: proxy-policy guarantee.}
By definition of the duality gap,
\[
\Gap(\bar z_T)
=
\max_{x'\in X}F(x',\bar y_T)-\min_{y'\in Y}F(\bar x_T,y').
\]
Since \(\max_{x'}F(x',\bar y_T)\ge F(\bar x_T,\bar y_T)=F(\bar z_T)\), we obtain
\[
\Gap(\bar z_T)\ge F(\bar z_T)-\min_{y'\in Y}F(\bar x_T,y').
\]
Rearranging,
\begin{equation}
\min_{y'\in Y}F(\bar x_T,y')\ge F(\bar z_T)-\Gap(\bar z_T).
\label{eq:minY-lower-delta}
\end{equation}

Now
\[
g_{\mathrm{proxy}}^\delta(\bar\pi_T;\rho)
=
\max_{\gamma\in\Gamma}\min_{y\in Y}F((\bar\pi_T,\gamma),y).
\]
Since \(\bar x_T=(\bar\pi_T,\bar\gamma_T)\in X\),
\[
g_{\mathrm{proxy}}^\delta(\bar\pi_T;\rho)\ge \min_{y\in Y}F(\bar x_T,y).
\]
Combining with \eqref{eq:minY-lower-delta} and then using \((28)\),
\begin{equation}
g_{\mathrm{proxy}}^\delta(\bar\pi_T;\rho)
\ge
F(\bar z_T)-\Gap(\bar z_T)
\ge
F(\bar z_T)-B_{T,\delta}.
\label{eq:proxy-certificate-delta}
\end{equation}

\medskip
\noindent
\textbf{Step 3: transfer from the proxy game to the original hard objective.}
By Theorem~\ref{thm:duality}, under Assumption~\ref{ass:geometry} we have the pointwise bounds
\[
0\le g_H^\delta(\pi;\rho)-g_{\mathrm{proxy}}^\delta(\pi;\rho)\le \rho\gamma_{\min},
\qquad
0\le V^\delta(\rho)-V_{\mathrm{proxy}}^\delta(\rho)\le \rho\gamma_{\min},
\qquad
\forall \pi\in\Pi_\epsilon.
\]
By Theorem~1, the truncation error satisfies
\[
0\le g_H^\delta(\pi;\rho)-g_H(\pi;\rho)\le \delta,
\qquad
0\le V^\delta(\rho)-V(\rho)\le \delta,
\qquad
\forall \pi\in\Pi_\epsilon.
\]

From \eqref{eq:proxy-certificate-delta},
\[
g_{\mathrm{proxy}}^\delta(\bar\pi_T;\rho)\ge F(\bar z_T)-B_{T,\delta}.
\]
Hence
\[
g_H(\bar\pi_T;\rho)
\ge
g_{\mathrm{proxy}}^\delta(\bar\pi_T;\rho)-\delta
\ge
F(\bar z_T)-B_{T,\delta}-\delta.
\]
This gives the explicit robust-value certificate.

Next, using \(V(\rho)\le V^\delta(\rho)\) and
\(g_H(\bar\pi_T;\rho)\ge g_{\mathrm{proxy}}^\delta(\bar\pi_T;\rho)-\delta\), we get
\begin{align*}
V(\rho)-g_H(\bar\pi_T;\rho)
&\le
V^\delta(\rho)-g_{\mathrm{proxy}}^\delta(\bar\pi_T;\rho)+\delta \\
&=
\bigl(V^\delta(\rho)-V_{\mathrm{proxy}}^\delta(\rho)\bigr)
+
\bigl(V_{\mathrm{proxy}}^\delta(\rho)-g_{\mathrm{proxy}}^\delta(\bar\pi_T;\rho)\bigr)
+\delta.
\end{align*}
By Theorem~\ref{thm:duality},
\[
V^\delta(\rho)-V_{\mathrm{proxy}}^\delta(\rho)\le \rho\gamma_{\min}.
\]
Also,
\[
V_{\mathrm{proxy}}^\delta(\rho)
=
\max_{\pi\in\Pi_\epsilon}g_{\mathrm{proxy}}^\delta(\pi;\rho)
=
\max_{x\in X}\min_{y\in Y}F(x,y),
\]
so
\[
V_{\mathrm{proxy}}^\delta(\rho)\le \max_{x\in X}F(x,\bar y_T).
\]
And from the definition of \(g_{\mathrm{proxy}}^\delta\),
\[
g_{\mathrm{proxy}}^\delta(\bar\pi_T;\rho)\ge \min_{y\in Y}F(\bar x_T,y).
\]
Subtracting the last two displays,
\[
V_{\mathrm{proxy}}^\delta(\rho)-g_{\mathrm{proxy}}^\delta(\bar\pi_T;\rho)
\le
\max_{x\in X}F(x,\bar y_T)-\min_{y\in Y}F(\bar x_T,y)
=
\Gap(\bar z_T)
\le
B_{T,\delta}.
\]
Combining the last three inequalities proves
\[
V(\rho)-g_H(\bar\pi_T;\rho)\le B_{T,\delta}+\rho\gamma_{\min}+\delta.
\]
Since \(B_{T,\delta}=\mathcal{O}(L_\delta/T)\), this also yields
\[
V(\rho)-g_H(\bar\pi_T;\rho)
=
O\!\left(\frac{L_\delta}{T}+\rho\gamma_{\min}+\delta\right).
\]

\medskip
\noindent
\textbf{Step 4: sharper guarantee under Assumption~\ref{ass:geometry}.}
Assume now Assumption~\ref{ass:geometry} in addition. Then by Theorem~\ref{thm:duality},
\[
g_{\mathrm{proxy}}^\delta(\pi;\rho)=g_H^\delta(\pi;\rho),
\qquad
V_{\mathrm{proxy}}^\delta(\rho)=V^\delta(\rho),
\qquad
\forall \pi\in\Pi_\epsilon.
\]
Therefore,
\[
V(\rho)-g_H(\bar\pi_T;\rho)
\le
V^\delta(\rho)-g_H(\bar\pi_T;\rho)
\le
V^\delta(\rho)-g_H^\delta(\bar\pi_T;\rho)+\delta.
\]
Using exact recovery and the proxy optimality bound already established,
\[
V^\delta(\rho)-g_H^\delta(\bar\pi_T;\rho)
=
V_{\mathrm{proxy}}^\delta(\rho)-g_{\mathrm{proxy}}^\delta(\bar\pi_T;\rho)
\le
B_{T,\delta}.
\]
Hence
\[
V(\rho)-g_H(\bar\pi_T;\rho)\le B_{T,\delta}+\delta.
\]
 Since \(B_{T,\delta}=\mathcal{O}(L_\delta/T)\), this gives the sharper rate
\[
V(\rho)-g_H(\bar\pi_T;\rho)
=
O\!\left(\frac{L_\delta}{T}+\delta\right).
\]
\end{proof}

\subsection{Computational Cost and Dependence on the Parameters}\label{app:cost}
This subsection makes the cost paragraph of Section~\ref{sec:algorithm} explicit.

\paragraph{One iteration.} An iteration of Algorithm~\ref{alg:four-player-ogda} computes the four block gradients of Lemma~\ref{lem:gradients} and takes the closed-form steps \eqref{eq:update-pi}--\eqref{eq:update-p}. The gradients need two matrix--vector products with the $m\times m$ matrix $P_p$, the $d=\binom m2$ weights $w_{ij}=\pi_i\nu_j-\pi_j\nu_i$, and one sum. The steps for $\pi$ and $\nu$ are a multiplicative update followed by a KL projection onto the floored simplex (one scalar bisection), the step for $p$ is coordinatewise, and $\gamma$ is clipped. Time and memory are $O(m^2)$ per iteration.

\paragraph{Number of iterations.} By Theorem~\ref{thm:main-gap}, $T=\lceil 8L_\delta(D_X^{\rm diam}+D_Y^{\rm diam})/\varepsilon_{\rm acc}\rceil$ iterations give $\Gap(\bar z_T)\le\varepsilon_{\rm acc}$, so the total work is $O\big(m^2L_\delta(D_X^{\rm diam}+D_Y^{\rm diam})/\varepsilon_{\rm acc}\big)$, where, by Lemma~\ref{lem:grad-var},
\[
L_\delta=\max\Big\{1,\ \frac{\tau}{\epsilon\,\pi_{\rm ref,min}},\ \frac{\tau}{\beta\,\pi_{\rm ref,min}},\ \frac{\gamma_{\max}}{\delta(1-\delta)},\ \sqrt d,\ 2\sqrt d\log\frac{1-\delta}{\delta}\Big\}.
\]
For $\pi_1=\nu_1=\pi_{\rm ref}$ and $p_1=p^\star\in P_\delta$ the diameters are bounded by
\[
D_X^{\rm diam}\le\log\frac1{\pi_{\rm ref,min}}+\frac12\max\{\gamma_{\max}-\gamma_1,\ \gamma_1-\gamma_{\min}\}^2,
\qquad
D_Y^{\rm diam}\le\log\frac1{\pi_{\rm ref,min}}+d\log\frac1\delta,
\]
because $\mathrm{KL}(\pi\|\pi_{\rm ref})\le\log(1/\pi_{\rm ref,min})$ and $\dBern(p\|q)\le\log\big(1/\min\{q,1-q\}\big)\le\log(1/\delta)$ for $q\in[\delta,1-\delta]$.

\paragraph{Dependence on the parameters.}
\emph{Truncation $\delta$.} $L_\delta$ grows like $1/\delta$ through the term $\gamma_{\max}/(\delta(1-\delta))$, while the truncation error is $\delta$. Balancing the two gives $\delta=O(T^{-1/2})$, as in Theorem~\ref{thm:main-gap}.

\emph{Radius $\rho$.} A dual maximizer lies in $[0,1/\rho]$ (Lemma~\ref{lem:gamma-upper}), so $\gamma_{\max}\ge1/\rho$ is enough (Assumption~\ref{ass:geometry}-(3)). If $\gamma_1\le\gamma_{\max}/2$, then $L_\delta D_X^{\rm diam}\ge\gamma_{\max}^3/(8\delta(1-\delta))\ge1/(8\rho^3\delta(1-\delta))$, so the bound on $T$ grows at least like $\rho^{-3}$ as $\rho\to0$. This is a property of the bound, not a lower bound for the algorithm, and the cost of one iteration does not depend on $\rho$.

\emph{Curvature floor $\gamma_{\min}$.} It enters the guarantee through the term $\rho\gamma_{\min}$, and Assumption~\ref{ass:geometry}-(4) requires $\tau\gamma_{\min}\ge\frac14$, so $\rho\gamma_{\min}\ge\rho/(4\tau)$.

\emph{Temperature $\tau$.} Assumption~\ref{ass:geometry}-(2) requires $\beta\le e^{-1/\tau}$, which places the follower's Gibbs best response in $\mathcal N_\beta$ (Step~1 of the proof of Theorem~\ref{thm:duality}). Hence $L_\delta\ge\tau/(\beta\pi_{\rm ref,min})\ge\tau e^{1/\tau}/\pi_{\rm ref,min}$, which grows exponentially as $\tau\to0$.

\section{Comparisons with Standard NLHF}\label{sec:vsnlhf}
Compared with standard NLHF \citep{munos2023nash} and its following studies \citep{Zhang2025INPO,tiapkin2025accelerating,nayak2025achieving,zhou2025extragradient,ye2024online}, the convergence analysis in our setting is substantially more
involved. In nominal NLHF, the preference kernel is fixed, and the optimization problem is a
regularized two-player zero-sum game between the main policy and the adversarial policy.
Once the game is written in saddle-point form, the KL regularizers provide strong concavity in
the leader policy and strong convexity in the follower policy, so one can analyze mirror descent
or optimistic no-regret dynamics directly on a fixed convex--concave game.
By contrast, our robust formulation introduces uncertainty in the preference kernel itself. The
resulting objective is no longer a standard two-player nominal game, but a distributionally
robust max-min problem in which the adversary chooses both a competing policy and a
plausible preference kernel. This creates two new difficulties that do not appear in standard
NLHF.

First, the original hard-constrained problem does not retain the clean convex-concave structure
of nominal NLHF. If one eliminates the follower using the Gibbs best response, the reduced
objective becomes a concave minimization problem over the kernel ambiguity set, which is
globally hard in general. On the other hand, if one keeps both the adversarial policy and the
preference kernel as optimization variables, the inner minimization is jointly nonconvex because
the win-rate term is bilinear in the pair \((\nu,p)\). Therefore, unlike in standard NLHF, one
cannot analyze the original hard-constrained robust objective directly using a standard minimax
or saddle-point argument.

A further difficulty is that the kernel block acquires its curvature through the dual variable
\(\gamma\). In the nominal two-player game, the KL regularizers already provide fixed curvature in
the policy blocks. In the four-player robust game, by contrast, the \(p\)-block is strongly convex
only with modulus proportional to \(\gamma\). This is why the lower bound \(\gamma_{\min}>0\) and
the condition \(\tau \gamma_{\min}\ge 1/4\) enter the analysis: they ensure that the KL curvature in
the kernel block is strong enough to dominate the mixed bilinear interaction between the
follower policy and the adversarial kernel. Without such a condition, the grouped
\((\pi,\gamma)\)-vs.-\((\nu,p)\) game need not be convex-concave, and the standard
optimistic mirror-descent analysis does not go through.

Our gradient-variation analysis is also substantially more coupled than in NLHF. In nominal
NLHF, once the preference kernel is fixed, the gradient of each player depends only on the other
policy. In our setting, the gradients evolve through a chain of dependencies: $\gamma \rightarrow p \rightarrow (\pi,\nu),$
so the variation of one block can depend on the movement of several other blocks. In
particular, the \(p\)-gradient depends on the current policies and on the dual variable, while the
\(\gamma\)-gradient is exactly the KL-budget slack of the current kernel. Establishing convergence
therefore requires a sharper self-bounding argument that simultaneously controls all four regret
terms and shows that the total gradient variation can be absorbed by the negative Bregman
stability terms. This absorption step has no analogue in the standard NLHF proof.

\section{Auxiliary Lemmas}
\label{app:aux_lemmas}
We prove several auxiliary lemmas required for the theorems.

\begin{lemma}[Universal upper bound on the dual multiplier]
\label{lem:gamma-upper}
Fix \((\pi,\nu)\in\Pi_\epsilon\times\Delta_m\), and define
\[
q_{\pi,\nu}(\gamma)
:=
\min_{p\in P_\delta}
\left\{
J(\pi,\nu;p)+\gamma\bigl(\BKL(p\|p^\star)-\rho\bigr)
\right\},
\qquad \gamma\ge 0.
\]
Then every maximizer of \(q_{\pi,\nu}\) over \([0,\infty)\) lies in \([0,1/\rho]\).
In particular, one may choose the algorithmic upper bound \(\gamma_{\max}\) so that
\(\gamma_{\max}\ge 1/\rho\).
\end{lemma}

\begin{proof}
Because \(p^\star\in P_\delta\), for every \(\gamma\ge 0\),
\[
q_{\pi,\nu}(\gamma)
\le
J(\pi,\nu;p^\star)-\rho\gamma.
\]
At \(\gamma=0\),
\[
q_{\pi,\nu}(0)=\min_{p\in P_\delta}J(\pi,\nu;p).
\]
For fixed \((\pi,\nu)\), the only part of \(J(\pi,\nu;p)\) that depends on \(p\) is the win-rate
term, which always lies in \([0,1]\). Hence
\[
J(\pi,\nu;p^\star)-q_{\pi,\nu}(0)\le 1.
\]
Therefore, if \(\gamma>1/\rho\), then
\[
q_{\pi,\nu}(\gamma)
\le
J(\pi,\nu;p^\star)-\rho\gamma
<
J(\pi,\nu;p^\star)-1
\le
q_{\pi,\nu}(0),
\]
so such a \(\gamma\) cannot be optimal.
\end{proof}

\begin{lemma}[Blockwise OMD inequalities]
\label{lem:block-omd}
For any comparators \(\pi\in\Pi_\epsilon\), \(\gamma\in\Gamma\), \(\nu\in\Nbeta\), and
\(p\in P_\delta\), the iterates produced by Algorithm \ref{alg:four-player-ogda} satisfy
\begin{align}
\sum_{t=1}^T \langle g_t^\pi,\pi-\pi_t\rangle
&\le
\frac{\KL(\pi\|\pi_1)}{\eta}
+
\eta\sum_{t=1}^T \|g_t^\pi-g_{t-1}^\pi\|_\infty^2
-
\frac{1}{2\eta}\sum_{t=1}^T \KL(\pi_t\|\pi_{t+1}),
\label{eq:block-pi}
\\
\sum_{t=1}^T g_t^\gamma(\gamma-\gamma_t)
&\le
\frac{(\gamma-\gamma_1)^2}{2\eta}
+
\eta\sum_{t=1}^T |g_t^\gamma-g_{t-1}^\gamma|^2
-
\frac{1}{2\eta}\sum_{t=1}^T |\gamma_t-\gamma_{t+1}|^2,
\label{eq:block-gamma}
\\
\sum_{t=1}^T \langle g_t^\nu,\nu_t-\nu\rangle
&\le
\frac{\KL(\nu\|\nu_1)}{\eta}
+
\eta\sum_{t=1}^T \|g_t^\nu-g_{t-1}^\nu\|_\infty^2
-
\frac{1}{2\eta}\sum_{t=1}^T \KL(\nu_t\|\nu_{t+1}),
\label{eq:block-nu}
\\
\sum_{t=1}^T \langle g_t^p,p_t-p\rangle
&\le
\frac{\BKL(p\|p_1)}{\eta}
+
\eta\sum_{t=1}^T \|g_t^p-g_{t-1}^p\|_2^2
-
\frac{1}{2\eta}\sum_{t=1}^T \BKL(p_t\|p_{t+1}).
\label{eq:block-p}
\end{align}
\end{lemma}

\begin{proof}
Apply Lemma \ref{lem:omd-standard} blockwise:
\begin{itemize}
    \item to the \(\pi\)-block with entropy geometry and maximization;
    \item to the \(\gamma\)-block with Euclidean geometry and maximization;
    \item to the \(\nu\)-block with entropy geometry and minimization;
    \item to the \(p\)-block with binary-entropy geometry and minimization.
\end{itemize}
The dual norms are \(\ell_\infty\) for \(\pi\) and \(\nu\), the Euclidean norm for \(\gamma\),
and \(\ell_2\) for \(p\).
\end{proof}

The constant-step OGDA proof requires a uniform bound on blockwise gradient variation.

\begin{lemma}[Gradient-variation bounds]
\label{lem:grad-var}
Let
\[
\pi_{\mathrm{ref},\min}:=\min_{1\le i\le m}\pi_{\mathrm{ref},i},
\qquad
L_\pi:=\frac{\tau}{\epsilon \pi_{\mathrm{ref},\min}},
\qquad
L_\nu:=\frac{\tau}{\beta \pi_{\mathrm{ref},\min}},
\qquad
L_p:=\frac{\gamma_{\max}}{\delta(1-\delta)},
\]
and
\[
M_\delta:=2\log\frac{1-\delta}{\delta}.
\]
For \(t\ge 2\), let
\[
z_t=(\pi_t,\gamma_t,\nu_t,p_t)\in X\times Y.
\]
Then the block gradients satisfy
\begin{align}
\|g_t^\pi-g_{t-1}^\pi\|_\infty
&\le
L_\pi \|\pi_t-\pi_{t-1}\|_1
+
\|\nu_t-\nu_{t-1}\|_1
+
\|p_t-p_{t-1}\|_2,
\label{eq:var-pi}
\\
|g_t^\gamma-g_{t-1}^\gamma|
&\le
\sqrt d\,M_\delta\,\|p_t-p_{t-1}\|_2,
\label{eq:var-gamma}
\\
\|g_t^\nu-g_{t-1}^\nu\|_\infty
&\le
\|\pi_t-\pi_{t-1}\|_1
+
L_\nu \|\nu_t-\nu_{t-1}\|_1
+
\|p_t-p_{t-1}\|_2,
\label{eq:var-nu}
\\
\|g_t^p-g_{t-1}^p\|_2
&\le
\sqrt d\,\|\pi_t-\pi_{t-1}\|_1
+
\sqrt d\,\|\nu_t-\nu_{t-1}\|_1
+
L_p \|p_t-p_{t-1}\|_2
+
\sqrt d\,M_\delta\,|\gamma_t-\gamma_{t-1}|.
\label{eq:var-p}
\end{align}
Consequently, if we define for \(t\ge 2\)
\[
S_t^\pi:=\|\pi_t-\pi_{t-1}\|_1^2,
\qquad
S_t^\gamma:=|\gamma_t-\gamma_{t-1}|^2,
\qquad
S_t^\nu:=\|\nu_t-\nu_{t-1}\|_1^2,
\qquad
S_t^p:=\|p_t-p_{t-1}\|_2^2,
\]
and
\[
L
:=
\max\left\{
1,\,
L_\pi,\,
L_\nu,\,
L_p,\,
\sqrt d,\,
\sqrt d\,M_\delta
\right\},
\]
then
\begin{align}
\|g_t^\pi-g_{t-1}^\pi\|_\infty^2
&\le
3L^2\bigl(S_t^\pi+S_t^\nu+S_t^p\bigr),
\label{eq:sq-var-pi}
\\
|g_t^\gamma-g_{t-1}^\gamma|^2
&\le
L^2 S_t^p,
\label{eq:sq-var-gamma}
\\
\|g_t^\nu-g_{t-1}^\nu\|_\infty^2
&\le
3L^2\bigl(S_t^\pi+S_t^\nu+S_t^p\bigr),
\label{eq:sq-var-nu}
\\
\|g_t^p-g_{t-1}^p\|_2^2
&\le
4L^2\bigl(S_t^\pi+S_t^\gamma+S_t^\nu+S_t^p\bigr).
\label{eq:sq-var-p}
\end{align}
\end{lemma}

\begin{proof}
We prove the first-order bounds and then square them.

\textbf{The \(\pi\)-block.}
From \eqref{eq:gpi} of Lemma \ref{lem:gradients},
\[
g^\pi(i)=\sum_{j=1}^m \nu_j P_p(i\succ j)-\tau\left(\log\frac{\pi_i}{\pi_{\mathrm{ref},i}}+1\right).
\]
For the win-rate component, since \(\sum_j \nu_j=1\),
\begin{align*}
&
\left|
\sum_j \nu_{t,j}P_{p_t}(i\succ j)-\sum_j \nu_{t-1,j}P_{p_{t-1}}(i\succ j)
\right|
\\
&\qquad\le
\left|
\sum_j (\nu_{t,j}-\nu_{t-1,j})P_{p_t}(i\succ j)
\right|
+
\left|
\sum_j \nu_{t-1,j}\bigl(P_{p_t}(i\succ j)-P_{p_{t-1}}(i\succ j)\bigr)
\right|
\\
&\qquad\le
\|\nu_t-\nu_{t-1}\|_1
+
\|p_t-p_{t-1}\|_\infty
\le
\|\nu_t-\nu_{t-1}\|_1+\|p_t-p_{t-1}\|_2.
\end{align*}
For the KL term, since every \(\pi_t(i)\ge \epsilon \pi_{\mathrm{ref},i}\ge \epsilon
\pi_{\mathrm{ref},\min}\),
\[
\left|
\log\frac{\pi_t(i)}{\pi_{\mathrm{ref},i}}
-
\log\frac{\pi_{t-1}(i)}{\pi_{\mathrm{ref},i}}
\right|
=
|\log \pi_t(i)-\log \pi_{t-1}(i)|
\le
\frac{|\pi_t(i)-\pi_{t-1}(i)|}{\epsilon \pi_{\mathrm{ref},\min}}
\]
by the mean-value theorem. Taking the maximum over \(i\) proves \eqref{eq:var-pi}.

\textbf{The \(\gamma\)-block.}
By \eqref{eq:ggamma} of Lemma \ref{lem:gradients},
\[
g^\gamma=\BKL(p\|p^\star)-\rho.
\]
For each coordinate \(e=(i,j)\in E\), the derivative of the scalar Bernoulli-KL term with respect
to \(p_e\) is
\[
h_e'(u)
=
\log\frac{u(1-p_e^\star)}{p_e^\star(1-u)}.
\]
Since both \(u\) and \(p_e^\star\) lie in \([\delta,1-\delta]\), we have
\[
|h_e'(u)|\le |\logit(u)|+|\logit(p_e^\star)|\le 2\log\frac{1-\delta}{\delta}=M_\delta.
\]
Therefore, by the mean-value theorem and Cauchy--Schwarz,
\[
|g_t^\gamma-g_{t-1}^\gamma|
\le
\sum_{e=1}^d M_\delta\,|p_{t,e}-p_{t-1,e}|
\le
\sqrt d\,M_\delta\,\|p_t-p_{t-1}\|_2,
\]
which proves \eqref{eq:var-gamma}.

\textbf{The \(\nu\)-block.}
The proof is identical to the \(\pi\)-block, now using the floor
\(\nu_{t,j}\ge \beta \pi_{\mathrm{ref},j}\ge \beta \pi_{\mathrm{ref},\min}\). This gives
\eqref{eq:var-nu}.

\textbf{The \(p\)-block.}
From \eqref{eq:gp} of Lemma \ref{lem:gradients},
\[
g^p_{ij}
=
w_{ij}(\pi,\nu)
+
\gamma\,h_{ij}(p_{ij}),
\qquad
h_{ij}(u):=\log\frac{u(1-p_{ij}^\star)}{p_{ij}^\star(1-u)}.
\]
For the antisymmetric weights, fix \((i,j)\in E\). Then
\begin{align*}
|w_{ij}(\pi_t,\nu_t)-w_{ij}(\pi_{t-1},\nu_{t-1})|
&=
|\pi_{t,i}\nu_{t,j}-\pi_{t,j}\nu_{t,i}
-\pi_{t-1,i}\nu_{t-1,j}
+\pi_{t-1,j}\nu_{t-1,i}|
\\
&\le
|\pi_{t,i}-\pi_{t-1,i}|
+
|\pi_{t,j}-\pi_{t-1,j}|
+
|\nu_{t,i}-\nu_{t-1,i}|
+
|\nu_{t,j}-\nu_{t-1,j}|
\\
&\le
\|\pi_t-\pi_{t-1}\|_1+\|\nu_t-\nu_{t-1}\|_1.
\end{align*}
Hence
\[
\|w(\pi_t,\nu_t)-w(\pi_{t-1},\nu_{t-1})\|_2
\le
\sqrt d\,\|\pi_t-\pi_{t-1}\|_1
+
\sqrt d\,\|\nu_t-\nu_{t-1}\|_1.
\]

For the KL-log term,
\[
|h_{ij}'(u)|=\frac{1}{u(1-u)}\le \frac{1}{\delta(1-\delta)},
\]
so
\[
\|h(p_t)-h(p_{t-1})\|_2
\le
\frac{1}{\delta(1-\delta)}\|p_t-p_{t-1}\|_2.
\]
Also \(|h_{ij}(u)|\le M_\delta\) on \([\delta,1-\delta]\), so
\[
\|h(p_{t-1})\|_2\le \sqrt d\,M_\delta.
\]
Therefore
\begin{align*}
\|g_t^p-g_{t-1}^p\|_2
&\le
\|w(\pi_t,\nu_t)-w(\pi_{t-1},\nu_{t-1})\|_2
+
\gamma_{\max}\|h(p_t)-h(p_{t-1})\|_2
+
|\gamma_t-\gamma_{t-1}|\,\|h(p_{t-1})\|_2
\\
&\le
\sqrt d\,\|\pi_t-\pi_{t-1}\|_1
+
\sqrt d\,\|\nu_t-\nu_{t-1}\|_1
+
L_p\|p_t-p_{t-1}\|_2
+
\sqrt d\,M_\delta\,|\gamma_t-\gamma_{t-1}|,
\end{align*}
which proves \eqref{eq:var-p}.

Finally, \eqref{eq:sq-var-pi}--\eqref{eq:sq-var-p} follow from
\[
(a+b+c)^2\le 3(a^2+b^2+c^2),
\qquad
(a+b+c+d)^2\le 4(a^2+b^2+c^2+d^2).
\]
\end{proof}

\begin{lemma}[Gradient variation is absorbed by the stability terms]
\label{lem:absorb}
Let
\[
V_T
:=
\eta\sum_{t=1}^T
\left(
\|g_t^\pi-g_{t-1}^\pi\|_\infty^2
+
|g_t^\gamma-g_{t-1}^\gamma|^2
+
\|g_t^\nu-g_{t-1}^\nu\|_\infty^2
+
\|g_t^p-g_{t-1}^p\|_2^2
\right),
\]
and
\[
\Xi_T
:=
\frac{1}{2\eta}\sum_{t=1}^T
\left(
\KL(\pi_t\|\pi_{t+1})
+
|\gamma_t-\gamma_{t+1}|^2
+
\KL(\nu_t\|\nu_{t+1})
+
\BKL(p_t\|p_{t+1})
\right).
\]
If
\begin{equation}
\eta\le \frac{1}{8L},
\label{eq:step-cond}
\end{equation}
where \(L\) is the constant from Lemma \ref{lem:grad-var}, then
\begin{equation}
V_T\le \Xi_T.
\label{eq:absorb-main}
\end{equation}
\end{lemma}

\begin{proof}
Because \(g_0^\bullet=g_1^\bullet\), the \(t=1\) contribution to \(V_T\) is zero. Thus
we sum from \(t=2\) onward. By Lemma \ref{lem:grad-var},
\[
V_T
\le
11L^2\eta\sum_{t=2}^T \bigl(S_t^\pi+S_t^\gamma+S_t^\nu+S_t^p\bigr).
\]
On the other hand, Pinsker's inequality and the Bernoulli Pinsker bound give
\[
\KL(\pi_t\|\pi_{t+1})\ge \frac12\|\pi_t-\pi_{t+1}\|_1^2=\frac12 S_{t+1}^\pi,
\]
\[
\KL(\nu_t\|\nu_{t+1})\ge \frac12\|\nu_t-\nu_{t+1}\|_1^2=\frac12 S_{t+1}^\nu,
\]
\[
\BKL(p_t\|p_{t+1})\ge 2\|p_t-p_{t+1}\|_2^2 = 2S_{t+1}^p.
\]
Hence
\begin{align*}
\Xi_T
&\ge
\frac{1}{4\eta}\sum_{t=1}^T S_{t+1}^\pi
+
\frac{1}{2\eta}\sum_{t=1}^T S_{t+1}^\gamma
+
\frac{1}{4\eta}\sum_{t=1}^T S_{t+1}^\nu
+
\frac{1}{\eta}\sum_{t=1}^T S_{t+1}^p
\\
&\ge
\frac{1}{4\eta}\sum_{t=2}^{T+1}\bigl(S_t^\pi+S_t^\gamma+S_t^\nu+S_t^p\bigr)
\\
&\ge
\frac{1}{4\eta}\sum_{t=2}^T \bigl(S_t^\pi+S_t^\gamma+S_t^\nu+S_t^p\bigr).
\end{align*}
Therefore \(V_T\le \Xi_T\) whenever
\[
11L^2\eta \le \frac{1}{4\eta},
\qquad\text{i.e.}\qquad
\eta \le \frac{1}{2\sqrt{11}\,L}.
\]
The stronger condition \eqref{eq:step-cond} is sufficient.
\end{proof}

\begin{lemma}[Standard optimistic mirror-descent bound]
\label{lem:omd-standard}
Let \(h\) be a mirror map with Bregman divergence \(D_h\), let \(\eta>0\), and let
\(\widehat g_t=2g_t-g_{t-1}\).
For minimization, the update
\[
z_{t+1}
=
\argmin_{z\in Z}
\left\{
\langle \widehat g_t,z\rangle+\frac{1}{\eta}D_h(z,z_t)
\right\}
\]
satisfies, for every comparator \(u\in Z\),
\[
\sum_{t=1}^T \langle g_t,z_t-u\rangle
\le
\frac{D_h(u,z_1)}{\eta}
+
\eta \sum_{t=1}^T \|g_t-g_{t-1}\|_*^2
-
\frac{1}{2\eta}\sum_{t=1}^T D_h(z_{t+1},z_t).
\]
For maximization, the analogous update with \(\argmax\) satisfies
\[
\sum_{t=1}^T \langle g_t,u-z_t\rangle
\le
\frac{D_h(u,z_1)}{\eta}
+
\eta \sum_{t=1}^T \|g_t-g_{t-1}\|_*^2
-
\frac{1}{2\eta}\sum_{t=1}^T D_h(z_{t+1},z_t).
\]
\end{lemma}

\begin{lemma}[Partial gradients]
\label{lem:gradients}
For \(z=(\pi,\gamma,\nu,p)\in X\times Y\), the block gradients of \(F\) are:
\begin{align}
g^\pi(i)
&=
\sum_{j=1}^m \nu_j\,P_p(i\succ j)
-
\tau\left(\log\frac{\pi_i}{\pi_{\mathrm{ref},i}}+1\right),
\qquad i=1,\dots,m,
\label{eq:gpi}
\\
g^\gamma
&=
\BKL(p\|p^\star)-\rho,
\label{eq:ggamma}
\\
g^\nu(j)
&=
\sum_{i=1}^m \pi_i\,P_p(i\succ j)
+
\tau\left(\log\frac{\nu_j}{\pi_{\mathrm{ref},j}}+1\right),
\qquad j=1,\dots,m,
\label{eq:gnu}
\\
g^p_{ij}
&=
w_{ij}(\pi,\nu)
+
\gamma\log\frac{p_{ij}(1-p^\star_{ij})}{p^\star_{ij}(1-p_{ij})},
\qquad (i,j)\in E.
\label{eq:gp}
\end{align}
\end{lemma}

\begin{proof}
Differentiate \eqref{eq:F-def}. The \(\pi\)- and \(\nu\)-gradients are obtained from the linear
win-rate term and the KL regularizers. The \(\gamma\)-gradient is immediate because
\(F\) is affine in \(\gamma\). For \eqref{eq:gp}, the coefficient of \(p_{ij}\) in the win-rate
decomposition is \(w_{ij}(\pi,\nu)\), and
\[
\frac{\partial}{\partial p_{ij}}
\BKL(p\|p^\star)
=
\log\frac{p_{ij}(1-p^\star_{ij})}{p^\star_{ij}(1-p_{ij})}.
\qedhere
\]
\end{proof}

\input{Experiments_imdb_appendix}
\input{Experiments_llm_appendix}

\newpage

\end{document}

%% file: exp_new.tex
\newcommand{\piref}{\pi_{\mathrm{ref}}}
\newcommand{\piL}{\pi_{\mathrm{L}}}
\newcommand{\piF}{\pi_{\mathrm{F}}}
\newcommand{\Pstar}{P^{*}}
\newcommand{\E}{\mathbb{E}}

\newcommand{\DKL}{D_{\mathrm{KL}}}
\newcommand{\Dbern}{D_{\mathrm{KL}}^{Bern}}
\newcommand{\sigmoid}{\sigma}
\newcommand{\clamp}{\mathrm{clip}}

We then evaluate our approach across three settings of increasing scale:
(i)~a synthetic tabular game that isolates the fragile-consensus phenomenon and has a known optimal robust value, run with robust training that solves the inner problem exactly,
(ii)~a best-of-$m$ reranking task on the IMDb dataset with heterogeneous annotator groups and explicitly perturbed kernels, run with a nested variant of OGDA (Appendix~\ref{app:imdb_optim}), and
(iii)~a summarization task with LLM-generated responses and external LLM judgments, run with the sampled analogue of Section~\ref{sec:algorithm}.
None of them verifies the guarantees of Section~\ref{sec:algorithm}.

\subsection{Tabular Game}
\label{sec:exp-tabular}

We first consider a collection of independently generated tabular prompts, each with a finite response set of size $m = 3$.
For each prompt~$x$, we construct a \emph{fragile-consensus} preference kernel by drawing a slim margin $s_x \sim \mathrm{Unif}[0.015, 0.040]$ and a solid margin $S_x \sim \mathrm{Unif}[0.25, 0.40]$, and setting
\begin{equation}\label{eq:kernel-tabular}
P^\star_x \;=\;
\begin{pmatrix}
  1/2 & 1/2 + s_x & 1/2 - S_x \\[3pt]
  1/2 - s_x & 1/2 & 1/2 + S_x \\[3pt]
  1/2 + S_x & 1/2 - S_x & 1/2
\end{pmatrix}.
\end{equation}
We then apply a random permutation of the response labels so that the identity of the fragile action varies across prompts.
This construction produces a cyclic preference game with one easily reversible edge (KL cost $\approx 2s_x^2 \le 0.003$) and two expensive-to-reverse edges (KL cost $\approx 2S_x^2/(S_x(1-S_x)) \ge 0.5$), which is precisely the regime in which robustness matters most: the adversary can flip the cheap edge within a small budget $\rho$, reversing the dominance ordering, while the expensive edges are untouchable.
We use a uniform reference policy and set $\tau = 0.05$.

We compare:
(i)~\textbf{NashMD}, the nominal two-player baseline operating on the fixed kernel $P^\star_x$ (optimistic mirror updates of the two policies, with the kernel frozen);
(ii)~\textbf{Four-player robust training} with training radii $\rho \in \{0.10, 0.20, 0.40\}$; and
(iii)~the \textbf{optimal robust value} $V(\rho)$, computed by golden-section search over the leader policy ($g_H$ is concave in $\pi$).
Both start from $\pi_1=\pi_{\rm ref}$ (NashMD also from $\nu_1=\pi_{\rm ref}$) and run for 2000 iterations; the step size is $0.30$ for NashMD and $0.05$ for robust training. Robust training is not covered by the guarantees of Section~\ref{sec:algorithm}.
We evaluate $g_H(\bar\pi_k;\rho)$, the worst-case value over the ball of radius $\rho$, at the averaged policy $\bar\pi_k=\frac1k\sum_{t\le k}\pi_t$.

Figure~\ref{fig:tabular}(a) shows $g_H(\bar\pi_k;0.2)$ over the iterations.
NashMD levels off about 0.032 below the optimal robust value $V(0.2)$, confirming that the nominal Nash policy is suboptimal for the robust objective.
The robust policy at $\rho = 0.2$ gets close to $V(0.2)$ (mean gap 0.0014 at iteration 2000).
At $\rho = 0.1$ (under-budgeted), the policy ends at a lower value, since it optimizes for a smaller perturbation set.
At $\rho = 0.4$ (over-budgeted), the policy is more conservative and also ends below the $\rho = 0.2$ value.
Figure~\ref{fig:tabular}(b) shows the paired per-prompt difference to NashMD of the final policies, with paired-bootstrap $95\%$ intervals over the $60$ prompts. The advantage does not simply grow with $\rho_{\rm eval}$: it peaks somewhat above the training radius (at $0.3$, $0.3$ and $0.45$ for $\rho_{\rm train}=0.1,0.2,0.4$) and then decreases slightly, and the policy trained at $\rho_{\rm train}=0.4$ is worse than NashMD on average at $\rho_{\rm eval}=0.05$ and $0.10$ (better on 0\% and 25\% of the prompts), as is the one trained at $\rho_{\rm train}=0.2$ at $\rho_{\rm eval}=0.05$. At the matched radius $\rho_{\rm eval}=\rho_{\rm train}$ the robust policy improves $g_H$ on all $60$ prompts, by 0.0160, 0.0309 and 0.0490 for $\rho_{\rm train}=0.1,0.2,0.4$ (half-widths of the intervals at most 0.0016), closing 91\%, 96\% and 97\% of the gap $V-g_H(\mathrm{NashMD})$, and it improves $g_H$ on all $60$ prompts at every larger evaluation radius tested.
This further confirms the enhanced robustness of our approach against preference uncertainty.

\begin{figure}[ht]
    \centering
    \par
        \includegraphics[width=\linewidth]{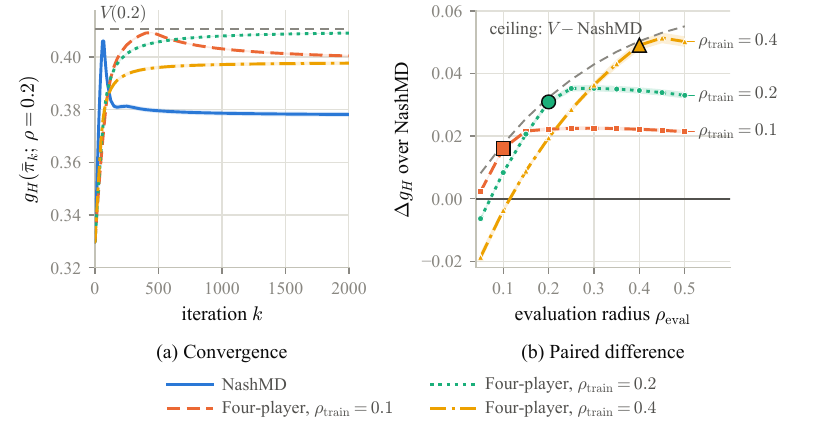}
        \caption{Tabular games. (a) $g_H(\bar\pi_k;0.2)$ of NashMD and of robust training with $\rho_{\rm train}\in\{0.1,0.2,0.4\}$ (running averages; mean and $95\%$ bootstrap interval over prompts) against $V(0.2)$. (b) Paired difference to NashMD of the final averaged policies versus the evaluation radius ($95\%$ paired-bootstrap intervals); large markers: matched radius; dashed: ceiling $V(\rho_{\rm eval})-g_H(\mathrm{NashMD})$.}
    \label{fig:tabular}
\end{figure}

\subsection{Best-of-$m$ IMDb Continuation Task}\label{sec:exp-imdb}
We next evaluate the method on a best-of-$m$ reranking task, where candidate responses are generated by a pretrained LLM and preference kernels are constructed from heterogeneous annotator groups.

For each of $N = 8000$ prompts from the IMDb review dataset~\citep{maas-EtAl:2011:ACL-HLT2011}, we pre-generate $m = 8$ candidate continuations using \texttt{google/flan-t5-large}~\citep{flan}. We generate the preferences based on three annotator groups:  \textbf{Comprehensive/analytical $(p_1)$} prefers detailed, well-structured responses;  \textbf{Efficient/concise $(p_2)$} prefers short, direct responses; \textbf{Balanced/nuanced $(p_3)$} prefers hedged, multi-perspective responses (scoring rules in Appendix~\ref{app:imdb_kernel}). We set the nominal kernel as the uniform mixture of the three preferences: $P_{p^\star}(a\succ b)=(\sum^3_{i=1} P_i(a\succ b))/3$, and construct the uncertainty set centered at it as $\mathcal{P}_{\rho_{\mathrm{eval}}}\triangleq \{ (w_1,w_2,w_3)\in\Delta_3: D_{\text{KL}}((w_1,w_2,w_3)||(1/3,1/3,1/3))\leq \rho_{\mathrm{eval}}\}$ to model preference uncertainty.

We then train and compare the robust policy $\pi_R$ (trained at radius $\rho_{\mathrm{train}}$) against the NashMD baseline $\pi_N$ by their pairwise \emph{win rate} under perturbed kernels.
For each evaluation radius $\rho_{\mathrm{eval}}$, we sample $M = 600$ kernels $\{P_j\}_{j=1}^M$ uniformly from $\mathcal{P}_{\rho_{\mathrm{eval}}}$ and compute the win rate:
\begin{equation}\label{eq:winrate-metric}
  \mathrm{WinRate}(\pi_R \text{ vs } \pi_N;\, \rho_{\mathrm{eval}})
  \;:=\; \bigl|\{j \in [M] : P_{p_j}(\pi_R \succ \pi_N) \ge 0.5\}\bigr|\big/M,
\end{equation}
where $P_{p_j}(\pi_R \succ \pi_N) = \sum_{i,k} \pi_R(i)\, \pi_N(k)\, p_j(i \succ k)$ is the probability that a response drawn from $\pi_R$ is preferred over one drawn from $\pi_N$ under kernel $p_j$.
A win rate above $50\%$ means the robust policy is preferred to NashMD under the majority of plausible kernels. Table~\ref{tab:imdb-winrate} (Appendix~\ref{app:imdb_eval}) reports the win rate of the robust policy against NashMD, averaged over all $8000$ prompts, for training radii $\rho_{\mathrm{train}} \in \{0.01, 0.02, 0.05, 0.10, 0.15\}$ and evaluation radii $\rho_{\mathrm{eval}} \in \{0.05, 0.10, 0.20, 0.30, 0.50\}$.

All $25$ entries of Table~\ref{tab:imdb-winrate} are at least $0.766$, also far above the training radius; within a column the win rate decreases as $\rho_{\rm train}$ grows ($0.986$ to $0.934$ at $\rho_{\rm eval}=0.05$). In plain terms, a small training radius already protects against kernels far from the nominal one (Table~\ref{tab:imdb-winrate}, $\tau=0.01$: $0.986$ to $0.844$ as $\rho_{\rm eval}$ grows from $0.05$ to $0.50$), and hedging harder buys nothing on these kernels. The gain depends on the KL temperature (Appendix~\ref{app:imdb_eval}) and is largest at $\tau=0.05$ (above $97\%$ for $\rho_{\rm eval}\le0.30$): at very small $\tau$ NashMD is already nearly deterministic, at large $\tau$ both policies are nearly uniform.

\subsection{Summarization Task with LLMs}
\label{sec:exp-tldr}
We further evaluate our performance on the TL;DR summarization task~\citep{stiennon2020learning}. The nominal preference kernel is a non-Bradley-Terry (non-BT) pairwise judge, $P^{\mathrm{pair}}_\phi(y_a\succ y_b\mid x)=\sigma\big(L^{\mathrm{pair}}(x,y_a,y_b)\big)$, instantiated with PairRM~\citep{jiang2023llmblender}: since each candidate's score is conditioned on its opponent, the kernel is not constrained to admit any scalar reward decomposition. On real model outputs, the closest Bradley-Terry fit to this kernel leaves a mean residual of $0.384$ logits per edge ($3.10$ percentage points), with $1.4\%$ of triples outright cyclic, confirming the nominal kernel used here is measurably non-BT (construction and diagnostic in Appendix~\ref{app:llm_nonbt}).
To evaluate robustness under preference uncertainty, we use three external LLMs from different model families: \texttt{DeepSeek-R1-Distill-Qwen-32B}~\citep{deepseek2025}, \texttt{Gemma-4-31B}~\citep{gemmateam2024}, and \texttt{Nemotron-3-Nano-Omni}~\citep{nvidia2026nemotron3nanoomni}, as judgment at evaluation time.  Since each external judge will induce its own pairwise preference (from their own reasoning) that differs from the nominal one, these judgments simulate a deployment scenario where the user population's preferences diverge from the training-time annotation distribution.

The results are presented in Table~\ref{tab:tldr-external}. Each entry is the share of evaluation prompts on which the judge prefers the robust policy's summary to NashMD's; the judges never entered training, so a win means the policy holds up under a preference it was not trained on. Most external judges (five of six configurations) yield a win rate above $0.50$ over NashMD, confirming that the robust policy remains competitive under this non-BT nominal kernel even though, unlike prior work, our framework does not require a Bradley-Terry reward model. The five wins range from $0.577$ to $0.732$, the sixth entry ($0.492$) is a statistical tie, and no training radius is best for every judge: the smaller wins under DeepSeek-R1, the larger under Gemma-4 and Nemotron-3. Appendix~\ref{app:llm_bt_extra} reports additional results with a Bradley-Terry nominal kernel across both fine-tuning backbones.

\vspace{-0.15in}
\begin{table}[!htbp]
\centering
\caption{Win rate against NashMD under external LLM judges, evaluated on $1000$ TL;DR prompts, with a non-BT (PairRM) nominal kernel.}\label{tab:tldr-external}
\renewcommand{\arraystretch}{1.2}
\begin{tabular}{ccccc}
\toprule
\multirow{2}{*}[-0.5ex]{\textbf{Fine-Tuning Model}} & \multirow{2}{*}[-0.5ex]{$\rho_{\text{train}}$} & \multicolumn{3}{c}{\textbf{Win Rate under Evaluation Model}} \\
\cmidrule(lr){3-5}
& & DeepSeek-R1 & Gemma-4 & Nemotron-3 \\
\midrule
\multirow{2}{*}{Qwen-1.5B} & 0.01 & \textbf{0.708} & \textbf{0.577} & 0.492 \\
                           & 0.05 & \textbf{0.625} & \textbf{0.615} & \textbf{0.732} \\
\bottomrule
\end{tabular}
\end{table}

%% file: Experiments_imdb_appendix.tex
\section{IMDb Continuation Task: Experiment Details}
\label{app:imdb_details}

\subsection{Nominal Preference and Annotator Groups}
\label{app:imdb_kernel}

We use the IMDb review corpus~\citep{maas-EtAl:2011:ACL-HLT2011} as a source of
continuation prompts.
Each prompt is the first sentence of a review; the task is to continue it.
We sample $N=8{,}000$ prompts and generate $m=8$ candidate continuations per prompt
using \texttt{google/flan-t5-large}~\citep{flan} with diverse instruction templates
and sampling temperatures.
The full set of prompt-level preference kernels is pre-computed and cached so that
NashMD and all robust variants are trained and evaluated on identical finite action sets.

To model heterogeneous human preferences we construct three annotator groups:
\textbf{Group~A} (comprehensive/analytical) prefers longer, more detailed continuations with
stronger evidence of analysis;
\textbf{Group~B} (efficient/concise) prefers short, direct, and decisive continuations; and
\textbf{Group~C} (balanced/nuanced) prefers responses that acknowledge both strengths and
weaknesses.
For prompt $x$ and group $g\in\{A,B,C\}$, each candidate $i$ receives a scalar score
$s_i^{(g)}$.
These scores are converted into a group-specific Bradley--Terry pairwise preference kernel:
\begin{equation}
  P_{ij}^{(g)} = \sigma\!\bigl(\beta\,(s_i^{(g)} - s_j^{(g)})\bigr),\quad i < j,
  \label{eq:imdb_bt}
\end{equation}
with $P_{ji}^{(g)}=1-P_{ij}^{(g)}$ and $P_{ii}^{(g)}=\tfrac{1}{2}$.
Writing $p^{(g)}\in[0,1]^d$, $d=\binom{m}{2}=28$, for the upper-triangular vector, the
nominal preference kernel is
$p^\star = \sum_{g\in\{A,B,C\}} \alpha_g\, p^{(g)}$
with fixed mixture weights $\alpha$.
Perturbing the mixture weights away from $\alpha$ via a Dirichlet draw produces
\emph{group-shift} kernels that lie in the convex hull of the group kernels and serve
as a model of structured population-level preference shift (Section~\ref{app:imdb_eval}).

\subsection{Optimization Algorithms}
\label{app:imdb_optim}

\paragraph{NashMD baseline.}
The NashMD baseline solves the nominal KL-regularized game
\begin{equation}
  \max_{\pi\in\Delta^m}\min_{\nu\in\Delta^m}
  J(\pi,\nu;P^\star) = \pi^\top P^\star \nu - \tau\,\KL(\pi\|\pi_{\mathrm{ref}})
  + \tau\,\KL(\nu\|\pi_{\mathrm{ref}})
  \label{eq:imdb_nominal}
\end{equation}
with nominal kernel $P^\star=P(p^\star)$ fixed throughout training.
We use simultaneous Optimistic Gradient Descent-Ascent (OGDA) on the policy pair
$(\pi,\nu)$ with step size $\eta_{\mathrm{Nash}}=0.30$.

\paragraph{Four-player robust algorithm.}
The robust variant hedges against kernel misspecification within the Bernoulli-KL ball
$\mathcal{U}_\rho=\{p:\BKL(p\|p^\star)\le\rho\}$
by solving
\begin{equation}
  \max_{\pi\in\Delta^m}
  \min_{\substack{p:\,\BKL(p\|p^\star)\le\rho\\p\in[\delta,1-\delta]^d}}
  \min_{\nu\in\Delta^m}
  J(\pi,\nu;P(p)).
  \label{eq:imdb_robust}
\end{equation}
The inner minimization is solved approximately by alternating between the
closed-form follower best-response and the kernel update
$p_{ij}\leftarrow\sigma\!\bigl(\logit(p^\star_{ij})
  - w_{ij}(\pi,\nu)/\!\max(\gamma,\gamma_{\min})\bigr)$
for $K=8$ inner steps, then the outer $(\pi,\gamma)$ update is performed with
OMD at step sizes $(\eta_\pi,\eta_\gamma)=(0.25,0.15)$.

\subsection{Hyperparameters}
\label{app:imdb_hparams}
\begin{table}[h]
\centering
\small
\caption{Hyperparameters for the IMDb continuation experiments.}
\label{tab:imdb_hparams}
\begin{tabular}{ll}
\toprule
\textbf{Component} & \textbf{Value} \\
\midrule
Candidates per prompt $m$ & $8$ \\
Kernel dimension $d$ & $\binom{8}{2}=28$ \\
Number of prompts $N$ & $8{,}000$ \\
Generator model & \texttt{flan-t5-large} \\
KL temperature $\tau$ & $\{0.001,0.005,0.01,0.05,0.1\}$ \\
NashMD step size $\eta_{\mathrm{Nash}}$ & $0.30$ \\
Robust step sizes $(\eta_\pi,\eta_\gamma)$ & $(0.25,\,0.15)$ \\
Inner steps $K$ & $8$ \\
Training radii $\rho_{\mathrm{train}}$ & $\{0.005,\,0.01,\,0.02,\,0.05,\,0.10,\,0.15\}$ \\
Evaluation radii $\rho_{\mathrm{eval}}$ & $\{0.05,\,0.10,\,0.15,\,0.20,\,0.30,\,0.50,\,0.75,\,1.0\}$ \\
Kernel clipping $\delta$ & $10^{-3}$ \\
Dual range $[\gamma_{\min},\gamma_{\max}]$ & $[0.02,\,500]$ \\

\bottomrule
\end{tabular}
\end{table}

\subsection{Results}
\label{app:imdb_eval}

Given the robust policy $\pi_R$ and the NashMD baseline $\pi_N$, we measure the
win rate using~\ref{eq:winrate-metric}.
The sampled kernels for evaluation are (i) the KL-ball boundary $\mathcal{U}_{\rho_{\mathrm{eval}}}$ sampled
uniformly at random, used for Table~\ref{tab:imdb-winrate}, and (ii) a Dirichlet-perturbed annotator mixture
$\tilde{p}=\sum_g\tilde{\alpha}_g p^{(g)}$,
$\tilde{\alpha}\sim\mathrm{Dirichlet}(\kappa\alpha)$, projected to the KL-ball
boundary at radius $\rho_{\mathrm{eval}}$, used in Table~\ref{tab:imdb_group_probwin}. Annotator group-shift kernels remain in the low-dimensional mixture subspace, making
them a more structured perturbation family.

\begin{table}[!ht]
\label{tab:imdb_probwin}
\centering
\caption{Win rate against NashMD under random kernel perturbations, averaged over $N = 8000$ IMDb prompts with $m = 8$ candidates.
Each entry reports $\mathrm{WinRate}(\pi_R \text{ vs } \pi_N;\, \rho_{\mathrm{eval}})$ from~\eqref{eq:winrate-metric}.
Values above $0.50$ indicate the robust policy is preferred under the majority of sampled kernels.}\label{tab:imdb-winrate}
\begin{tabular}{@{}r ccccc@{}}
\toprule
& \multicolumn{5}{c}{Evaluation radius $\rho_{\mathrm{eval}}$} \\
\cmidrule(l){2-6}
$\rho_{\mathrm{train}}$ & $0.05$ & $0.10$ & $0.20$ & $0.30$ & $0.50$ \\
\midrule
0.01  & 0.986 & 0.967 & 0.927 & 0.894 & 0.844 \\
0.02  & 0.985 & 0.965 & 0.925 & 0.891 & 0.841 \\
0.05  & 0.984 & 0.961 & 0.918 & 0.883 & 0.832 \\
0.10  & 0.978 & 0.949 & 0.901 & 0.866 & 0.814 \\
0.15  & 0.934 & 0.894 & 0.844 & 0.810 & 0.766 \\
\bottomrule
\end{tabular}
\end{table}

\begin{table}[h]
\centering
\small
\caption{%
  $\mathrm{WinRate}(\pi_R \text{ vs } \pi_N;\, \rho_{\mathrm{eval}})$ under group-shift kernels at $\tau=0.01$,
  averaged over $8{,}000$ prompts.
  Bold marks the highest value in each column.%
}
\label{tab:imdb_group_probwin}
\begin{tabular}{@{}r|rrrrr@{}}
\toprule
$\rho_{\mathrm{train}}$ & $\rho_{\mathrm{eval}}=0.05$ & $0.10$ & $0.20$ & $0.30$ & $0.50$ \\
\midrule
0.005 & \textbf{0.974} & \textbf{0.945} & \textbf{0.897} & \textbf{0.863} & \textbf{0.818} \\
0.01  & 0.974 & 0.944 & 0.897 & 0.862 & 0.818 \\
0.02  & 0.973 & 0.943 & 0.896 & 0.862 & 0.818 \\
0.05  & 0.970 & 0.939 & 0.893 & 0.859 & 0.816 \\
0.10  & 0.962 & 0.928 & 0.884 & 0.851 & 0.811 \\
0.15  & 0.917 & 0.882 & 0.840 & 0.811 & 0.775 \\
\bottomrule
\end{tabular}
\end{table}

Table~\ref{tab:imdb_tau_probwin} reports $\mathrm{WinRate}(\pi_R \text{ vs } \pi_N;\, \rho_{\mathrm{eval}})$ at the best
training radius for each KL-regularization temperature
$\tau\in\{0.001,0.005,0.01,0.05,0.1\}$.
The strongest overall regime is $\tau=0.05$: the robust policy wins with probability
above $97\%$ for $\rho_{\mathrm{eval}}\le 0.30$.
At very small $\tau$ the NashMD baseline is already highly concentrated, reducing the
marginal gain from robust training; at large $\tau$ both policies approach the uniform
distribution and the adversary has limited room to exploit kernel perturbations.

\begin{table}[h]
\centering
\small
\caption{%
  $\mathrm{WinRate}(\pi_R \text{ vs } \pi_N;\, \rho_{\mathrm{eval}})$ at the best $\rho_{\mathrm{train}}$ for each $\tau$
  (KL-random kernels).%
}
\label{tab:imdb_tau_probwin}
\begin{tabular}{@{}rr|rrrrrr@{}}
\toprule
$\tau$ & best $\rho_{\mathrm{tr}}$ & $\rho_{\mathrm{eval}}=0.05$ & $0.10$ & $0.20$ & $0.30$ & $0.50$ & $1.0$ \\
\midrule
0.001 & 0.005 & 0.972 & 0.944 & 0.896 & 0.859 & 0.810 & 0.735 \\
0.005 & 0.005 & 0.980 & 0.956 & 0.911 & 0.877 & 0.826 & 0.750 \\
0.01  & 0.005 & 0.986 & 0.967 & 0.929 & 0.896 & 0.847 & 0.768 \\
0.05  & 0.10  & \textbf{0.996} & \textbf{0.995} & \textbf{0.988} & \textbf{0.977} & \textbf{0.948} & \textbf{0.879} \\
0.1   & 0.10  & 0.945 & 0.931 & 0.903 & 0.880 & 0.841 & 0.775 \\
\bottomrule
\end{tabular}
\end{table}

Figure~\ref{fig:imdb_probwin_tau05} shows the full win-rate curves at $\tau=0.05$.
Training radii $\rho_{\mathrm{train}}\le 0.10$ are statistically indistinguishable from one another at
every evaluation radius (their pairwise differences never exceed the 95\% confidence half-width), while
$\rho_{\mathrm{train}}=0.15$ is significantly worse, with the gap widening as $\rho_{\mathrm{eval}}$ grows.

\begin{figure}[ht]
\centering
\par
\includegraphics[width=\linewidth]{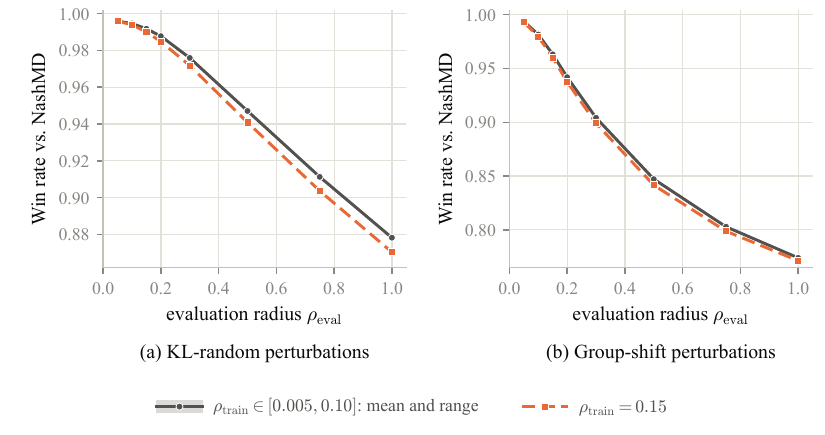}
\caption{%
  $\mathrm{WinRate}(\pi_R \text{ vs } \pi_N;\, \rho_{\mathrm{eval}})$ at $\tau=0.05$, for each training radius
  $\rho_{\mathrm{train}}\in\{0.005,0.01,0.02,0.05,0.10,0.15\}$. Radii $\rho_{\mathrm{train}}\in[0.005,0.10]$
  (gray band and mean) are statistically tied at every evaluation radius; $\rho_{\mathrm{train}}=0.15$ (orange)
  is significantly worse, and increasingly so as $\rho_{\mathrm{eval}}$ grows.%
}
\label{fig:imdb_probwin_tau05}
\end{figure}

Table~\ref{tab:imdb_tau_rho_detail} gives the full $(\tau,\rho_{\mathrm{train}})$
interaction at $\rho_{\mathrm{eval}}=0.30$.
For $\tau\le 0.01$ the smallest training radius is uniformly best.
For $\tau\ge 0.05$ the optimum shifts to $\rho_{\mathrm{train}}=0.10$, though the
improvement over smaller radii is modest.

\begin{table}[h]
\centering
\small
\caption{%
  $\mathrm{WinRate}(\pi_R \text{ vs } \pi_N;\, \rho_{\mathrm{eval}})$ at $\rho_{\mathrm{eval}}=0.30$,
  by $(\tau,\rho_{\mathrm{train}})$ (KL-random kernels).
  Bold marks the row-wise maximum.%
}
\label{tab:imdb_tau_rho_detail}
\begin{tabular}{@{}r|rrrrrr@{}}
\toprule
$\tau \;\backslash\; \rho_{\mathrm{tr}}$
  & 0.005 & 0.01 & 0.02 & 0.05 & 0.10 & 0.15 \\
\midrule
0.001 & \textbf{0.860} & 0.857 & 0.850 & 0.832 & 0.813 & 0.780 \\
0.005 & \textbf{0.877} & 0.874 & 0.870 & 0.857 & 0.831 & 0.784 \\
0.01  & \textbf{0.896} & 0.895 & 0.893 & 0.883 & 0.866 & 0.810 \\
0.05  & 0.976 & 0.976 & 0.976 & 0.976 & \textbf{0.977} & 0.972 \\
0.1   & 0.859 & 0.860 & 0.862 & 0.868 & \textbf{0.880} & 0.717 \\
\bottomrule
\end{tabular}
\end{table}

%% file: Experiments_llm_appendix.tex
\section{TL;DR Summarization Task: Experiment Details}
\label{app:llm_details}
\subsection{From Algorithm~\ref{alg:four-player-ogda} to the sampled implementation}
\label{app:llm_bridge}
For a language model the objects of Algorithm~\ref{alg:four-player-ogda}, sums over $m$ responses, simplex mirror steps and the exact kernel tilt, are not available, because a response is a token sequence. We follow the route of Munos et al.~\cite{munos2023nash} (their Appendix~F): sample a pair instead of summing, replace each proximal step by one stochastic policy-gradient step, and keep the reference KL inside the reward as a log-ratio, which is the score-function form of $\nabla_\theta\mathrm{KL}(\pi_\theta\|\piref)$ because $\mathbb E\,\nabla_\theta\log\pi_\theta=0$. Their opponent, a token-level geometric mixture of the policy and the reference, is replaced here by the trained follower $\nu_{\theta'}$: at the nominal equilibrium the follower's best response is the leader itself, so a mixture of leader and reference is a good stand-in, whereas under a perturbed kernel the best response is a different policy. We write $\pi_\theta,\nu_{\theta'}$ for the leader and follower that Appendices~\ref{app:llm_adversary}--\ref{app:llm_losses} call $\pi_{\rm L},\pi_{\rm F}$, and $\dBern(a\|b)=a\log\frac ab+(1-a)\log\frac{1-a}{1-b}$ for the Bernoulli KL of a single comparison.

\begin{algorithm}[t]
\caption{Sampled four-player step}
\label{alg:sampled}
\begin{algorithmic}[1]
\Require leader $\pi_\theta$, follower $\nu_{\theta'}$ (separate LoRA adapters, both initialized at $\piref$); head $g_\psi$ (zero-initialized output layer); multiplier $\gamma\ge0$; baselines $b_L,b_F$; $\tau$, $\rho$, $\gamma_{\max}$, step sizes $\eta_\theta,\eta_{\theta'},\eta_\gamma$, batch size $B$, head steps $K$, clip level $c$, EMA weight $\lambda$
\For{each iteration}
  \State draw prompts $x_1,\dots,x_B$; sample $y_L^i\sim\pi_\theta(\cdot\mid x_i)$ and $y_F^i\sim\nu_{\theta'}(\cdot\mid x_i)$
  \State $q_i\gets P^\star(y_L^i\succ y_F^i\mid x_i)$ (BT judge or PairRM); \; $\kappa_L^i\gets\clip\!\big(\log\tfrac{\pi_\theta(y_L^i\mid x_i)}{\piref(y_L^i\mid x_i)},\pm c\big)$, \; $\kappa_F^i$ likewise
  \State $p_i\gets p_i(\psi)=\sigma\big(\logit q_i+\delta_\psi(y_L^i,y_F^i)\big)$ \Comment{head as at the start of the step}
  \State $R^L_i\gets p_i-\tau\kappa_L^i$, \; $R^F_i\gets-p_i-\tau\kappa_F^i$ \Comment{$p_i$ held constant}
  \State $b_L\gets\lambda b_L+(1-\lambda)\tfrac1B\sum_iR_i^L$, \; $b_F\gets\lambda b_F+(1-\lambda)\tfrac1B\sum_iR_i^F$
  \State $\hat g_\theta\gets\tfrac1B\sum_i(R^L_i-b_L)\,\nabla_\theta\log\pi_\theta(y_L^i\mid x_i)$, \; $\hat g_{\theta'}\gets\tfrac1B\sum_i(R^F_i-b_F)\,\nabla_{\theta'}\log\nu_{\theta'}(y_F^i\mid x_i)$
  \For{$k=1,\dots,K$} \Comment{same batch, no new samples}
    \State gradient step on $\psi$ (descent) for $\mathcal L_\psi=\tfrac1B\sum_ip_i(\psi)+\gamma\big(\tfrac1B\sum_i\dBern(p_i(\psi)\|q_i)-\rho\big)$
  \EndFor
  \State $\gamma\gets\clip\big(\gamma+\eta_\gamma(\tfrac1B\sum_i\dBern(p_i(\psi)\|q_i)-\rho),\,0,\,\gamma_{\max}\big)$ with $\psi$ after the head steps
  \State $\theta\gets\theta+\eta_\theta\hat g_\theta$, \; $\theta'\gets\theta'+\eta_{\theta'}\hat g_{\theta'}$ \Comment{Adam, gradient-norm clipping}
\EndFor
\end{algorithmic}
\end{algorithm}

\textbf{What the sampled method approximates.}
\emph{(i) Noise:} each update uses one comparison per prompt and a small batch, so it is a noisy estimate of the corresponding block gradient of the tabular game. The mirror steps of Algorithm~\ref{alg:four-player-ogda} move the logits additively, whereas a gradient step on the adapter parameters weights each response by its probability.
\emph{(ii) Baselines:} the leader and the follower subtract exponential-moving-average baselines, which need no value function.
\emph{(iii) KL proxy:} $\kappa$ is a clipped sequence-level log-ratio, which biases the KL term whenever the clip is active.
\emph{(iv) Optimization:} a few gradient steps of the head stand in for the exact kernel tilt; there is no optimism and no averaging; $\gamma$ is clipped to $[0,\gamma_{\max}]$, so Assumption~\ref{ass:geometry}-(3) and (4) need not hold.
\emph{(v) Start:} both adapters start at $\piref$, so the two policies coincide at the start and, in expectation, the kernel player has no incentive to move until they separate; early training behaves like the nominal algorithm.

\subsection{Fine-Tuning Models.}
\label{app:llm_finetune}
We initialize the models from \texttt{Qwen2.5-1.5B-Instruct} \cite{qwen2} and \texttt{Gemma-2-2B} \cite{gemmateam2024} as both the initial and the reference model across the respective experiments and fine-tune with LoRA~\cite{lora2021} for 5000 gradient steps with KL regularization $\tau = 0.005$. The four-player game formulation was compared with adversarial budget $\rho \in \{0.01, 0.05\}$. The leader and follower policies carry separate LoRA adapters attached to the same frozen backbone.

\subsection{Nominal Preference Proxy}
\label{app:llm_nominal}
We train a Bradley-Terry(BT) preference judge model to serve as the nominal preference proxy. The judge is a scalar score model $s_{\phi}:(x,y) \to \mathbb{R}$ built on \texttt{DistilBERT-base-uncase} \cite{Sanh2019DistilBERTAD} backbone with a linear scalar head, and fine-tuned with the standard BT cross-entropy loss on the TL;DR dataset:

\begin{equation}
 \mathcal{L}_{\mathrm{BT}}(\phi) = -\mathbb{E}_{(x, y_w, y_l)} \left[ \log \sigma \big( s_{\phi}(x, y_w) - s_{\phi}(x, y_l) \big) \right],
\end{equation}

where $(y_w, y_l)$ denotes the preferred and dispreferred responses respectively. Training uses 83,573 comparison pairs. After 3 epochs the judge achieves 72\% validation accuracy. During training with both NashMD and Four-Player the pairwise preference is calculated as:

\begin{equation}
  P^{*}(y_a \succ y_b \mid x) = \sigma\big(s_{\phi}(x, y_a) - s_{\phi}(x, y_b)\big)
\end{equation}

\subsection{Adversary Proxy}
\label{app:llm_adversary}
The kernel-adversary in the four-player game is parameterized by an \emph{antisymmetric} perturbation head
$\delta_\psi(y_a, y_b)$ of the nominal preference logit.
The head computes
\begin{equation}
  \delta_\psi(y_a, y_b)
  \;=\; g_\psi(u_{ab}) - g_\psi(u_{ba}),
  \qquad
  g_\psi : \mathbb{R}^7 \to \mathbb{R},
  \label{eq:antisym}
\end{equation}
where $g_\psi$ is a two-layer MLP:
$\mathbb{R}^7 \xrightarrow{\;W_1\;} \mathbb{R}^{64}
 \xrightarrow{\;\tanh\;} \mathbb{R}^{64}
 \xrightarrow{\;W_2\;} \mathbb{R}$,
with the output layer $W_2$ initialized to zero so that
$\delta_\psi(y_a, y_b) = 0$ (no perturbation) at initialization. The algorithm therefore initializes at the nominal preference and gradually grows the adversarial perturbation.
Antisymmetry in \eqref{eq:antisym} ensures $\delta_\psi(y_a, y_b) = -\delta_\psi(y_b, y_a)$
and hence $P_\psi(y_a \succ y_b) + P_\psi(y_b \succ y_a) = 1$ at all iterates.

The perturbed preference is
\begin{equation}
  P_\psi(y_a \succ y_b \mid x)
  \;=\; \sigmoid\!\bigl(\logit\Pstar(y_a \succ y_b \mid x)
        + \delta_\psi(y_a, y_b)\bigr).
  \label{eq:perturbed_pref}
\end{equation}

The Bernoulli-KL ambiguity constraint $\mathbb{E}\bigl[D_{\mathrm{KL}}^{\mathrm{Bern}}(p_\psi\,\|\,p^\star)\bigr]\le \rho$ is enforced softly through the dual variable $\gamma$ in the Lagrangian payoff. We additionally clip the predicted logit shift before applying the sigmoid in~\ref{eq:perturbed_pref} to prevent $P_\psi$ from saturating to $\{0,1\}$ to stabilize the dual ascent.

\subsection{Losses.}
\label{app:llm_losses}
For a mini-batch with samples $y_L \sim \piL$ and $y_F \sim \piF$, we define
$p_\psi = P_\psi(y_L \succ y_F \mid x)$ and KL proxies
$\kappa_L = \clamp(\log\piL(y_L)/\piref(y_L), -c, c)$,
$\kappa_F = \clamp(\log\piF(y_F)/\piref(y_F), -c, c)$. The three player losses are
\begin{align}
  \mathcal{L}_\psi &=
    \E\bigl[p_\psi\bigr]
    + \gamma\,\bigl(\E[\Dbern(p_\psi \,\|\, \Pstar)] - \rho\bigr),
    \label{eq:adv_loss}
  \\[4pt]
  \mathcal{L}_\pi &=
    -\E\bigl[(p_\psi - \tau\kappa_L - b_L)\,\log\piL(y_L)\bigr],
    \label{eq:leader_loss}
  \\[4pt]
  \mathcal{L}_\nu &=
    -\E\bigl[(-p_\psi - \tau\kappa_F - b_F)\,\log\piF(y_F)\bigr],
    \label{eq:follower_loss}
\end{align}
where $b_L$, $b_F$ are EMA baselines for variance reduction, and the adversary $\psi$ \emph{minimizes}
\eqref{eq:adv_loss} (gradient descent) while the dual variable is updated by clipped projected gradient ascent:
\begin{equation}
  \gamma \;\leftarrow\; \clamp\!\Bigl(
    \gamma + \eta_\gamma\,\bigl(\E[\Dbern] - \rho\bigr),\;
    0,\; \gamma_{\max}
  \Bigr).
  \label{eq:dual_update}
\end{equation}

Each outer step interleaves $K$ adversary updates per $(\pi_L,\pi_F,\gamma)$ update. We default to $K=3$ as we obseved $K=1$ is too weak (the constraint is far from active) while $K \geq3$ tracks the constraint budget closely.

\subsection{Evaluation with LLM Judges}
\label{app:llm_judges_eval}
We employ three comparatively larger open-source models---DeepSeek-R1-Distill-Qwen-32B~\citep{deepseek2025}, Gemma-4-31B~\citep{gemmateam2024}, and Nemotron-3-Nano-Omni~\citep{nvidia2026nemotron3nanoomni}---to obtain preference signals querying with the following prompt template:

\begin{quote}
\small\texttt{%
You are an expert judge evaluating the quality of two generated summaries.\\
Post to summarize: ⟨text⟩\\
Summary A: ⟨summary1⟩\\
Summary B: ⟨summary2⟩\\
Evaluate which summary is better based on: (1) Accuracy, (2) Coherence, (3) Conciseness, (4) Helpfulness.\\
}
\end{quote}

where $\langle\text{text}\rangle$ corresponds to the Reddit post $x$, $\langle\text{summary1}\rangle$ to the policy response $y \sim \pi(\cdot \mid x)$, and $\langle\text{summary2}\rangle$ to the baseline response $y' \sim \pi_{\text{base}}(\cdot \mid x)$. We estimate the preference probability $P^{\text{ext}}_j(\pi \succ \pi_{\text{base}})$ as the fraction of comparisons in which the policy output is preferred over $1000$ prompts from the TL;DR dataset. Following Munos et al.~\cite{munos2023nash}, each comparison is treated as a Bernoulli trial and uncertainty is quantified using Clopper–Pearson 95\% confidence intervals; for $n=1000$, the worst-case half-width is approximately $\pm 0.032$. Four player policies preserve strong nominal performance against the NashMD baseline. Training against the preference uncertainty has not collapsed into overly conservative solutions, rather, the robust policies still improve the nominal pairwise objective with the pairwise judges across evaluation criteria.

\subsection{Non-BT Nominal Kernel: Construction and Diagnostic}
\label{app:llm_nonbt}
The TL;DR experiment in Section~\ref{sec:exp-tldr} uses a non-Bradley--Terry (non-BT)
nominal kernel rather than the standard, reproducible Bradley--Terry (BT) construction
(BT judge, Appendix~\ref{app:llm_nominal})~\citep{stiennon2020learning}, which is antisymmetric
by design: $P^\star(y_a\succ y_b\mid x)=\sigma\big(s_\phi(x,y_a)-s_\phi(x,y_b)\big)$.
The ambiguity set $\mathcal{P}$ in Eq.~\eqref{eq:amb set} is a Bernoulli-KL ball over the
$d=\binom{m}{2}$ pairwise probabilities themselves (imposed in expectation over
sampled pairs in TL;DR, Appendix~\ref{app:llm_losses}), with no constraint of the form
$\mathrm{logit}(p_{ij})=s_i-s_j$: even from a BT center, the worst-case
kernel found by the adversary may already be cyclic or otherwise not representable by
any scalar reward. We go further and construct a nominal kernel that is itself non-BT,
described next, together with a diagnostic confirming it is measurably so.

\textbf{Non-BT nominal kernel.} We replace the BT judge with a directly
parameterized pairwise kernel,
\[
P^{\mathrm{pair}}_\phi(y_a\succ y_b\mid x)=\sigma\big(L^{\mathrm{pair}}(x,y_a,y_b)\big),
\qquad
L^{\mathrm{pair}}(x,y_a,y_b)=\tfrac{c}{2}\big[q_\phi(x,y_a,y_b)-q_\phi(x,y_b,y_a)\big],
\]
where $q_\phi$ is a single joint forward pass over the post and both candidate
summaries, instantiated with PairRM~\citep{jiang2023llmblender} (a 0.4B
DeBERTa-v3-large cross-encoder). Each
candidate's representation is conditioned on its opponent, so its score is not a
function of $(x,y)$ alone and the kernel is not constrained to be BT. Because
$q_\phi$ is order-dependent, the difference of the two orderings is what gives
$L^{\mathrm{pair}}(a,b)=-L^{\mathrm{pair}}(b,a)$, hence
$P^{\mathrm{pair}}(a\succ b)+P^{\mathrm{pair}}(b\succ a)=1$ exactly. Here $c>0$ is a scalar temperature fit by maximum likelihood on TL;DR human comparisons.
The policy is
Qwen2.5-1.5B-Instruct. The adversary head, Bernoulli-KL constraint, dual update and
NashMD baseline are as in Appendix~\ref{app:llm_adversary}--\ref{app:llm_losses}; the
only change is that three of the adversary's seven input features, which used
per-response rewards, now carry the pairwise logit.

\textbf{The kernel is not BT.} Let
$L_{ij}=\mathrm{logit}\,P^{\mathrm{pair}}_\phi(y_i\succ y_j\mid x)$. A BT kernel has
$L_{ij}=s_i-s_j$, so for every triple $C:=L_{ij}+L_{jk}+L_{ki}=0$: the scores
telescope away, for any reward model. The least-squares BT fit to a triple removes
$C/3$ from each edge, so $|C|/3$ is the per-edge distance to the nearest BT kernel.
On 1{,}000 triples of real model outputs (per prompt, one summary from each of
NashMD-PG and the two robust policies trained under this kernel), the mean best-BT
projection error is $0.384$ logits per edge ($3.10$ percentage points in preference
probability), the median $|C|$ is $0.96$, and $1.4\%$ of triples are outright
cyclic. Edges are
confident ($|L|\approx3$ logits on average), so a $0.38$-logit residual rarely flips
a sign: cycles are rare, yet the departure from BT is present in nearly every
triple. The nominal kernel used in training is therefore measurably non-BT.

\subsection{Additional Results with a Bradley-Terry Nominal Kernel}
\label{app:llm_bt_extra}
For completeness, we also report results under the standard Bradley--Terry (BT)
nominal kernel (Appendix~\ref{app:llm_nominal}) that Section~\ref{sec:exp-tldr}'s
non-BT experiment departs from. Table~\ref{tab:tldr_bt_extra} reports the win rate
of the robust policy against NashMD on the same protocol as
Table~\ref{tab:tldr-external} (three external judges, $1000$ TL;DR prompts,
Clopper--Pearson half-width $\approx\pm0.032$ at $n=1000$), and additionally covers
a second fine-tuning backbone, Gemma-2B, alongside Qwen-1.5B.

\begin{table}[!htbp]
\centering
\caption{Win rate against NashMD under external LLM judges, evaluated on $1000$ TL;DR prompts, with a Bradley-Terry nominal kernel (cf.\ Table~\ref{tab:tldr-external}, which uses a non-BT kernel).}\label{tab:tldr_bt_extra}
\renewcommand{\arraystretch}{1.2}
\begin{tabular}{ccccc}
\toprule
\multirow{2}{*}[-0.5ex]{\textbf{Fine-Tuning Model}} & \multirow{2}{*}[-0.5ex]{$\rho_{\text{train}}$} & \multicolumn{3}{c}{\textbf{Win Rate under Evaluation Model}} \\
\cmidrule(lr){3-5}
& & DeepSeek-R1 & Gemma-4 & Nemotron-3 \\
\midrule
\multirow{2}{*}{Qwen-1.5B} & 0.01 & 0.496 & \textbf{0.506} & \textbf{0.764} \\
                           & 0.05 & \textbf{0.574} & \textbf{0.593} & \textbf{0.568} \\
\midrule
\multirow{2}{*}{Gemma-2B}  & 0.01 & \textbf{0.542} & \textbf{0.518} & 0.462 \\
                           & 0.05 & \textbf{0.640} & \textbf{0.720} & \textbf{0.583} \\
\bottomrule
\end{tabular}
\end{table}

\newpage
\section{LLM Judge Reasoning Examples}
\label{app:judge_examples}

\definecolor{chartblue}{HTML}{2a78d6}    %
\definecolor{chartorange}{HTML}{eb6834}  %
\definecolor{chartaqua}{HTML}{1baf7a}    %
\definecolor{chartyellow}{HTML}{c98500}  %
\definecolor{chartviolet}{HTML}{4a3aa7}  %

\subsection*{Example 1 \quad (\texttt{prompt\_id}~57)}
\label{app:judge_ex_57}

\noindent\textit{Disagreement:} Nemotron-3-Nano $\to$ \textbf{A};\enspace
Gemma-4-31B $\to$ \textbf{B};\enspace
DeepSeek-R1-32B $\to$ \textbf{B}

\begin{tcolorbox}[colback=black!3, colframe=black!35,
  title={\small Post (r/relationships)}, fonttitle=\bfseries\small]
\small
\textit{Title:} Me [20m] just got out of an 8 month relationship with my gf [19F].
I am very emotional, she wasn't.

\smallskip
I am far from the stereotypical macho guy. I love romance~[\ldots]
She never speaks about her feelings, never told me she loved me without me having
to say it first~[\ldots]
My question aimed at both genders: Are girls generally emotional, do you ladies
like to speak about emotions?
Do you generally like when the guy is cute with you and calls you babe, and says
he loves you?~[\ldots]
\end{tcolorbox}

\begin{tcolorbox}[colback=chartblue!6, colframe=chartblue!55, coltitle=chartblue!80!black,
  title={\small Summary A (NashMD-PG)}, fonttitle=\bfseries\small]
\small
OP is seeking advice on whether girls generally like to be romantic and express
emotions, and if they prefer a partner who is affectionate and expresses love.
\end{tcolorbox}

\begin{tcolorbox}[colback=chartorange!6, colframe=chartorange!55, coltitle=chartorange!70!black,
  title={\small Summary B (Four-Player $\rho{=}0.05$)}, fonttitle=\bfseries\small]
\small
The OP is seeking to understand if girls generally express emotions and appreciate
romantic displays of affection, as he is a romantic guy who enjoys being cute and
expressing his love.
He's questioning if this is something most girls find appealing.
\end{tcolorbox}

\begin{tcolorbox}[colback=chartaqua!6, colframe=chartaqua!55, coltitle=chartaqua!70!black,
  title={\small Nemotron-3-Nano \textnormal{(verdict: \textbf{A}, 14.4\,s)}},
  fonttitle=\bfseries\small]
\small
\textit{[Conciseness]} A is more concise (fewer words).
B includes extra info about his own traits---possible unnecessary padding.
So A wins on conciseness. [\ldots]

\textit{[Helpfulness]} A directly states the request for advice.
B adds context that may help reader understand OP's perspective, but the key
point---seeking advice about girls' preferences---is stated more succinctly in
A. [\ldots]

Overall, Summary A seems better due to higher conciseness and directness while
still accurate.
\end{tcolorbox}

\begin{tcolorbox}[colback=chartyellow!6, colframe=chartyellow!55, coltitle=chartyellow!85!black,
  title={\small Gemma-4-31B \textnormal{(verdict: \textbf{B}, 67.3\,s)}},
  fonttitle=\bfseries\small]
\small
\textit{[Helpfulness]} Summary B captures the ``why'' of the post whereas A only
captures the ``what''.
The OP is not asking a generic question about gender; he is questioning his own
experience in light of his personality and recent breakup.
This makes B a more complete reflection of the source text. [\ldots]

Summary A is more concise; Summary B is more contextually complete without
becoming wordy.
The motivation behind the question is a key part of the post, so B is slightly
better overall.
\end{tcolorbox}

\begin{tcolorbox}[colback=chartviolet!6, colframe=chartviolet!55, coltitle=chartviolet!80!black,
  title={\small DeepSeek-R1-Distill-Qwen-32B \textnormal{(verdict: \textbf{B}, 58.6\,s)}},
  fonttitle=\bfseries\small]
\small
\textit{[Helpfulness]} Summary A captures the main question but lacks some
context.
Summary B provides more background, making it clearer why OP is asking this and
what he has experienced. [\ldots]

Summary B includes more details about OP's romantic nature and his girlfriend's
behaviour, which are crucial for understanding his perspective.
\end{tcolorbox}

\subsection*{Example 2 \quad (\texttt{prompt\_id}~125)}
\label{app:judge_ex_125}

\noindent\textit{Disagreement:} Nemotron-3-Nano $\to$ \textbf{A};\enspace
Gemma-4-31B $\to$ \textbf{B};\enspace
DeepSeek-R1-32B $\to$ \textbf{B}

\begin{tcolorbox}[colback=black!3, colframe=black!35,
  title={\small Post (r/dating\_advice)}, fonttitle=\bfseries\small]
\small
\textit{Title:} I (21,M) need help getting a girl (22) who used to like me
interested again.

\smallskip
I recently started messaging a girl I knew from high school and it seemed to be
going somewhere, but a couple of bad decisions and me not being honest enough in
my replies have led me into a rut.
She seems to have lost interest and only replies when I message first.
How do I get from that to a position where I can ask her out?~[\ldots]
\end{tcolorbox}

\begin{tcolorbox}[colback=chartblue!6, colframe=chartblue!55, coltitle=chartblue!80!black,
  title={\small Summary A (NashMD-PG)}, fonttitle=\bfseries\small]
\small
OP (21M) is trying to rekindle a past connection with a 22F he knows from high
school.
He's lost some mojo and is unsure how to move from a passive messaging pattern to
asking her out.
\end{tcolorbox}

\begin{tcolorbox}[colback=chartorange!6, colframe=chartorange!55, coltitle=chartorange!70!black,
  title={\small Summary B (Four-Player $\rho{=}0.05$)}, fonttitle=\bfseries\small]
\small
A 21-year-old man is trying to rekindle a connection with a 22-year-old woman he
knew from high school, but after some miscommunications, she's lost interest.
He's seeking advice on how to get back in her good graces and ask her out.
\end{tcolorbox}

\begin{tcolorbox}[colback=chartaqua!6, colframe=chartaqua!55, coltitle=chartaqua!70!black,
  title={\small Nemotron-3-Nano \textnormal{(verdict: \textbf{A}, 16.9\,s)}},
  fonttitle=\bfseries\small]
\small
\textit{[Helpfulness]} Summary A mentions the ``passive messaging pattern'',
capturing the specific current dynamic (she only replies when prompted).
Summary B says ``lost interest'' but does not identify this nuance. [\ldots]

Summary A captures the specific dynamic slightly better; final verdict A.
\end{tcolorbox}

\begin{tcolorbox}[colback=chartyellow!6, colframe=chartyellow!55, coltitle=chartyellow!85!black,
  title={\small Gemma-4-31B \textnormal{(verdict: \textbf{B}, 74.8\,s)}},
  fonttitle=\bfseries\small]
\small
\textit{[Accuracy]} Summary B captures the cause-effect relationship: bad
decisions $\to$ lost interest $\to$ needs advice.
Summary A's ``lost some mojo'' reflects a side detail from the Fun Facts section
rather than the primary driver of the conflict in the main text. [\ldots]

\textit{[Helpfulness]} Summary B provides a more complete narrative arc (the
attempt, the mistake, the result, the goal), giving better context for advice.
\end{tcolorbox}

\begin{tcolorbox}[colback=chartviolet!6, colframe=chartviolet!55, coltitle=chartviolet!80!black,
  title={\small DeepSeek-R1-Distill-Qwen-32B \textnormal{(verdict: \textbf{B}, 161.7\,s)}},
  fonttitle=\bfseries\small]
\small
\textit{[Accuracy]} Summary B includes specific reasons for her losing interest
(miscommunications / bad decisions), which aligns with the post.
Summary A is vague, saying only he ``lost some mojo'', which does not capture the
cause of her disengagement. [\ldots]

\textit{[Helpfulness]} B explains \textit{why} she lost interest, providing
context that would help readers offer more targeted advice.
\end{tcolorbox}

\subsection*{Example 3 \quad (\texttt{prompt\_id}~114)}
\label{app:judge_ex_114}

\noindent\textit{Disagreement:} Nemotron-3-Nano $\to$ \textbf{B};\enspace
Gemma-4-31B $\to$ \textbf{B};\enspace
DeepSeek-R1-32B $\to$ \textbf{A}

\begin{tcolorbox}[colback=black!3, colframe=black!35,
  title={\small Post (r/relationships)}, fonttitle=\bfseries\small]
\small
\textit{Title:} [m24] Arranged a date and it was postponed. Am I wasting my time?

\smallskip
A few weeks ago I met a girl, got her number and arranged to meet up.
The day before, she asked to postpone because of an exam.
I said sure; she then said she had exams for another couple of weeks.
After they finished, I texted her on Saturday night~[\ldots] and got the response
``no but~[\ldots]''
Am I wasting my time?
\end{tcolorbox}

\begin{tcolorbox}[colback=chartblue!6, colframe=chartblue!55, coltitle=chartblue!80!black,
  title={\small Summary A (NashMD-PG)}, fonttitle=\bfseries\small]
\small
OP is confused and questioning if he's wasting time after postponing a date with
a girl who initially agreed to meet up but then said she'd be busy for a few more
weeks.
\end{tcolorbox}

\begin{tcolorbox}[colback=chartorange!6, colframe=chartorange!55, coltitle=chartorange!70!black,
  title={\small Summary B (Four-Player $\rho{=}0.05$)}, fonttitle=\bfseries\small]
\small
The OP arranged a date with a girl but she kept postponing it due to exams,
leading him to question if he wasted his time.
He's confused by her lack of clear communication and is wondering if she's not
interested.
\end{tcolorbox}

\begin{tcolorbox}[colback=chartaqua!6, colframe=chartaqua!55, coltitle=chartaqua!70!black,
  title={\small Nemotron-3-Nano \textnormal{(verdict: \textbf{B}, 12.0\,s)}},
  fonttitle=\bfseries\small]
\small
\textit{[Accuracy]} Summary A says she ``initially agreed'' and was ``busy for a
few more weeks''---neither is quite right.
She did not explicitly agree before the postponement, and the exams are now
finished.
Summary B accurately states she postponed \textit{due to exams} and captures his
uncertainty about her interest. [\ldots]

\textit{[Helpfulness]} B conveys both the postponement reason and his underlying
concern (is she interested?), which is the real question driving the post.
\end{tcolorbox}

\begin{tcolorbox}[colback=chartyellow!6, colframe=chartyellow!55, coltitle=chartyellow!85!black,
  title={\small Gemma-4-31B \textnormal{(verdict: \textbf{B}, 80.1\,s)}},
  fonttitle=\bfseries\small]
\small
\textit{[Accuracy]} Summary A says ``after postponing a date''---grammatically
ambiguous, implying OP did the postponing, whereas the girl was the one who
postponed.
Summary B correctly identifies that \textit{she} kept postponing, and also
captures his underlying concern about her interest level. [\ldots]

\textit{[Helpfulness]} B explicitly states the core of OP's dilemma: whether she
is actually interested.
Summary A misses this psychological dimension entirely.
\end{tcolorbox}

\begin{tcolorbox}[colback=chartviolet!6, colframe=chartviolet!55, coltitle=chartviolet!80!black,
  title={\small DeepSeek-R1-Distill-Qwen-32B \textnormal{(verdict: \textbf{A}, 70.8\,s)}},
  fonttitle=\bfseries\small]
\small
\textit{[Accuracy]} Summary A captures the postponement and his confusion closely.
Summary B adds that he is confused by ``a lack of clear communication'', which is
an interpretation not explicitly stated in the post---the OP is confused about
whether she is interested, not about communication clarity per se. [\ldots]

\textit{[Conciseness]} Summary A sticks strictly to the facts from the post.
Summary B introduces an extra interpretive layer which might be considered
redundant.
\end{tcolorbox}

\subsection*{Example 4 \quad (\texttt{prompt\_id}~71)}
\label{app:judge_ex_71}

\noindent\textit{Disagreement:} Nemotron-3-Nano $\to$ \textbf{B};\enspace
Gemma-4-31B $\to$ \textbf{B};\enspace
DeepSeek-R1-32B $\to$ \textbf{A}

\begin{tcolorbox}[colback=black!3, colframe=black!35,
  title={\small Post (r/relationships)}, fonttitle=\bfseries\small]
\small
\textit{Title:} I [F21] feel frustrated with my boyfriend's [M24] laziness.
Am I justified?

\smallskip
I have a well-paying part-time job and am a full-time student; he works full time
but doesn't make ends meet and I often cover his expenses.
He plays video games while I'm trying to work on important assignments.
When I asked him for help with housework he got mad, left without saying goodbye,
and said \textit{we should just break up}~[\ldots]
Am I justified in feeling frustrated?
\end{tcolorbox}

\begin{tcolorbox}[colback=chartblue!6, colframe=chartblue!55, coltitle=chartblue!80!black,
  title={\small Summary A (NashMD-PG)}, fonttitle=\bfseries\small]
\small
OP (F21) is frustrated with her boyfriend's (M24) lack of help and appreciation,
especially during her stressful school deadlines.
She feels unappreciated and \textbf{wants to break up} because he doesn't
contribute to household chores or offer support.
\end{tcolorbox}

\begin{tcolorbox}[colback=chartorange!6, colframe=chartorange!55, coltitle=chartorange!70!black,
  title={\small Summary B (Four-Player $\rho{=}0.05$)}, fonttitle=\bfseries\small]
\small
The poster is frustrated with her boyfriend's lack of effort in contributing to
household responsibilities and feeling unappreciated despite her hard work and
stress.
She feels he's only concerned with his own leisure activities and doesn't offer
support during her school stress.
\end{tcolorbox}

\begin{tcolorbox}[colback=chartaqua!6, colframe=chartaqua!55, coltitle=chartaqua!70!black,
  title={\small Nemotron-3-Nano \textnormal{(verdict: \textbf{B}, 45.7\,s)}},
  fonttitle=\bfseries\small]
\small
\textit{[Accuracy]} Summary A says she ``wants to break up''; however, it was
\textit{he} who suggested breaking up, not her.
She expressed frustration and felt underappreciated, but never stated she wants
to end the relationship.
Summary A therefore introduces a claim not directly supported by the post. [\ldots]

Summary B stays strictly within what is stated: frustration, lack of effort in
household duties, feeling unappreciated, his focus on leisure.
Accuracy is the primary criterion; B is more faithful to the post.
\end{tcolorbox}

\begin{tcolorbox}[colback=chartyellow!6, colframe=chartyellow!55, coltitle=chartyellow!85!black,
  title={\small Gemma-4-31B \textnormal{(verdict: \textbf{B}, 72.2\,s)}},
  fonttitle=\bfseries\small]
\small
\textit{[Accuracy]} Summary A contains a significant factual error.
It states that OP ``wants to break up,'' whereas the text specifies that it was
the \textit{boyfriend} who said ``we should just break up'' during the argument.
Summary B is accurate and faithfully reflects the content of the post. [\ldots]

Summary A's inaccuracy regarding who suggested the break-up makes it unreliable;
B is the clear winner on accuracy.
\end{tcolorbox}

\begin{tcolorbox}[colback=chartviolet!6, colframe=chartviolet!55, coltitle=chartviolet!80!black,
  title={\small DeepSeek-R1-Distill-Qwen-32B \textnormal{(verdict: \textbf{A}, 58.6\,s)}},
  fonttitle=\bfseries\small]
\small
\textit{[Accuracy]} Summary A mentions that she wants to break up---a specific,
emotionally significant detail from the post.
Although it was he who voiced the suggestion, the context implies she may be
considering it given her frustration; this is arguably a reasonable inference
rather than a fabrication. [\ldots]

\textit{[Helpfulness]} Including the break-up detail gives a clearer picture of
the severity of her feelings, making Summary A more useful for understanding how
serious the situation is.
\end{tcolorbox}